\pdfoutput=1 
\documentclass[11pt]{article}

\usepackage[font=libertinus, citestyle=numeric]{kurbanlab}

\DeclareAffiliation{hbku}{%
  College of Science and Engineering, Hamad Bin Khalifa University, Doha, Qatar}

\DeclareAffiliation{tamu}{%
  Department of Computer and Electrical Engineering,
  Texas A\&M University, College Station, TX, USA}

\DeclareAffiliation{iub}{%
  Luddy School of Informatics, Computing, and Engineering,
  Indiana University Bloomington, Bloomington, IN, USA}

\DeclareAffiliation{uis}{%
  Department of Computer Science,
  University of Illinois Springfield, Springfield, IL, USA}

\DeclareMathOperator{\sg}{sg}
\DeclareMathOperator{\ReLU}{ReLU}

\title{LGQ: Learnable Geometric Quantization for Image Tokenization}
\Subtitle{A learnable codebook that never collapses: annealed soft assignment with a $K$-scaled diversity regularizer}
\RunningTitle{LGQ: Learnable Geometric Quantization for Image Tokenization}

\Author{Idil Bilge Altun}{iub}
\Author{Mert Onur Cakiroglu}{iub}
\Author{Elham Buxton}{uis}
\Author{Mehmet Dalkilic}{iub}
\Author[corresponding=hkurban@hbku.edu.qa, orcid=0000-0003-3142-2866]{Hasan Kurban}{hbku}

\CodeURL{https://anonymous.4open.science/r/lgq-anon-E12C/}
\Venue{Under review}

\begin{document}
\maketitle

\begin{abstract}
Recent collapse-free quantizers such as FSQ achieve stable training by replacing the
learnable codebook with an engineered geometry: a fixed scalar grid whose structure is
dictated by the codebook size $K$. We show this trade-off is unnecessary. We introduce
\textbf{Learnable Geometric Quantization (LGQ)}, which retains a learnable codebook of
codes and performs soft-to-hard assignment via temperature annealing, regularized by two
cheap terms: a \emph{diversity} term scaled by codebook size that penalizes concentrated
batch-average usage is the primary driver of collapse resistance, complemented by a
\emph{peakedness} term that sharpens each token's soft-assignment toward one-hot; together
they prevent codebook collapse without EMA, reset heuristics, or codebook
reparameterization. Under a fixed VQ-GAN backbone, we benchmark LGQ against RotVQ, FSQ,
LFQ, SimVQ, and IBQ on ImageNet $256{\times}256$ at $K=16{,}384$, and sweep LGQ over
$K\in\{4096,\dots,65{,}536\}$ without any per-$K$ hyperparameter tuning. LGQ attains the
best reconstruction FID at $K=16{,}384$ while maintaining $100\%$ codebook utilization,
and continues to improve as the codebook grows to $K=65{,}536$, holding $100\%$
utilization at every $K$. Training MaskGIT on the frozen tokenizers, LGQ further attains
the best class-conditional generation among the compared quantizers, leading on
reconstruction and generation alike. Code is available at
\url{https://anonymous.4open.science/r/lgq-anon-E12C/}.
\end{abstract}

\printkeywords

\begin{figure}[t]
    \centering
    \includegraphics[width=0.8\linewidth]{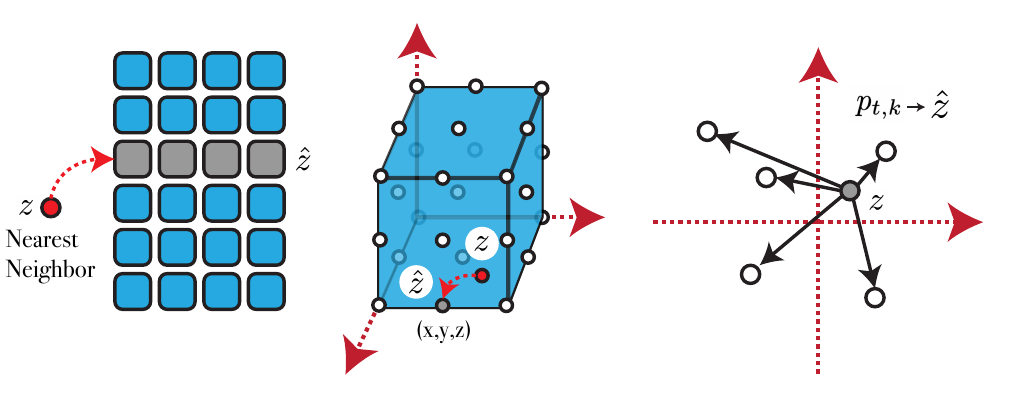}
    \caption{\textbf{VQ-VAE, FSQ and LGQ comparison.}
    Conceptual comparison of three discrete quantizers (VQ, FSQ, LGQ). VQ's flat learned
    codebook collapses; FSQ's fixed grid is collapse-free but unable to adapt to the data
    anisotropy; LGQ keeps a \emph{learnable} codebook with explicit usage regularization.}
    \label{fig:teaser}
\end{figure}

\section{Introduction}
\label{sec:intro}

Discrete representation learning underpins modern generative image pipelines:
vector-quantized autoencoders~\cite{vqvae,vqvae2} compress an image into a short sequence
of tokens that autoregressive~\cite{vqvae,maskgit} and masked-generative~\cite{maskgit,lfq}
models then learn to predict. This makes the tokenizer the ceiling of the entire system.
No downstream model can recover detail its tokenizer has discarded, and a degenerate token
distribution, one where a handful of codes absorb most of the usage, leaves the generator
with a vocabulary that is effectively far smaller than it appears. The central obstacle to
a good tokenizer is therefore keeping a large codebook fully and meaningfully used, the
problem this paper addresses.

Three classes of quantizer dominate the recent literature (\cref{fig:teaser}).
(i)~\textbf{Vector Quantization (VQ)}~\cite{vqvae} learns a flat codebook updated by
exponential moving averages, but is widely reported to undergo \emph{codebook
collapse}~\cite{vqvae2,simvq}: a small fraction of entries dominate.
(ii)~\textbf{Finite Scalar Quantization (FSQ)}~\cite{fsq} replaces the learned codebook
with fixed per-channel scalar grids, eliminating collapse by construction at the cost of
flexibility. (iii)~\textbf{Lookup-Free Quantization (LFQ)}~\cite{lfq} encodes each spatial
location as a binary string, making the codebook implicit. Two more recent designs like
SimVQ~\cite{simvq} that reparameterizes the codebook through a learned linear projection,
and the rotation trick~\cite{rotation_trick} that transforms gradients through quantization
attempt to fix VQ's optimization pathology directly. A common thread runs through the
recent collapse-resistant designs: they either (a) restrict where codes may be placed (FSQ,
LFQ rigidly, SimVQ and VQGAN-LC softly), (b) restrict how many codes receive gradient by
assigning hard (VQ, RotVQ), or (c) change the rule by which codes are scored (IBQ). LGQ
instead maintains a fully learnable codebook and pays for collapse resistance with two
cheap soft-assignment regularizers.

\paragraph{Contributions}
\begin{itemize}
\item We introduce \textbf{LGQ}, a learnable quantizer whose soft-assignment, temperature
schedule, and $K$-scaled diversity regularizer together eliminate codebook collapse with a
single fixed hyperparameter setting across a $16\times$ codebook-size range ($K{=}4{\rm K}$
to $K{=}65{\rm K}$) without EMA, reset, reparameterization, or per-scale structural
redesign, holding $100\%$ utilization at every $K$ where IBQ under the identical protocol
falls to $55\%$ at $65$K (\cref{sec:method,sec:results-sweep}).
\item We provide a controlled \textbf{$K=16{,}384$ comparison} of six quantizers on
ImageNet $256{\times}256$ under an identical VQ-GAN-style backbone and reconstruction
objective, in which LGQ achieves the best rFID, PSNR, SSIM, and LPIPS
(\cref{sec:results-main}).
\item We \textbf{ablate} the regularizer weights ($\lambda_{\text{peak}}$,
$\lambda_{\text{div}}$) and the temperature schedule ($\tau_{\text{start}}$,
$\tau_{\text{end}}$, annealing curve), isolating the design choices that drive LGQ's
collapse resistance (\cref{app:ablations}), and isolate the selection rule with a matched
logit ablation (\cref{sec:vs-ibq}).
\item We report a \textbf{generation-side comparison}: a small MaskGIT transformer trained
on each frozen tokenizer under a matched recipe, on which LGQ achieves the best
class-conditional gFID among the five tokenizers (\cref{tab:maskgit_gen}).
\end{itemize}

\section{Background}
\label{sec:background}

\paragraph{Discrete tokenizers} A discrete tokenizer encodes an image as a short sequence
of tokens from a fixed vocabulary, so that discrete sequence models, \eg autoregressive
transformers or MaskGIT~\cite{maskgit}, can be trained to generate images. It consists of
an encoder, quantizer, and decoder $(E, Q, D)$. The encoder produces a grid of continuous
embeddings $\mathbf{z} = E(x) \in \mathbb{R}^{B \times T \times C}$ ($B$ batch size, $T$
spatial tokens, $C$ embedding dimension); for a single token we write $\mathbf{z} \in
\mathbb{R}^{C}$. The quantizer holds a \emph{codebook} $\mathcal{C} =
\{\mathbf{e}_k\}_{k=1}^{K}$ and maps each token to an index $k = Q(\mathbf{z}) \in
\{1,\dots,K\}$ with embedding $\hat{\mathbf{z}} = \mathbf{e}_k$; the decoder reconstructs
$\hat x = D(\hat{\mathbf{z}})$. The methods differ almost entirely in $Q$; the encoder,
decoder, and reconstruction objective are shared.

\paragraph{Hard assignment and collapse} The canonical quantizer~\cite{vqvae} assigns each
token to its nearest code, $k^\star = \arg\min_k \lVert \mathbf{z} - \mathbf{e}_k\rVert_2^2$.
Since $\arg\min$ is non-differentiable, it is trained with a straight-through estimator
(STE)~\cite{ste} that copies the decoder gradient onto the encoder output, plus a
\emph{commitment} term $\beta\,\lVert \mathbf{z} - \sg[\mathbf{e}_{k^\star}]\rVert_2^2$
($\sg[\cdot]$ is stop-gradient) which keeps encoder outputs near their assigned codes.
Because assignment is hard, only the selected code per token receives a gradient;
rarely-chosen codes drift and stop being selected, causing \emph{codebook collapse}, where
many codes go unused (\emph{dead codes}) and the effective vocabulary is far smaller than
$K$. Standard fixes like EMA updates or dead-code resets treat the symptom rather than the
assignment rule. We measure this through \emph{utilization}, the fraction of codes used at
least once over an evaluation set; keeping it high without sacrificing the adaptivity of a
learnable codebook is the challenge the methods in \cref{sec:related} address, and the axis
along which we compare them with LGQ. Soft-assignment methods, the family LGQ belongs to,
route gradient to all $K$ codes at once (\cref{sec:method}).

\begin{figure}[t]
    \centering
    \includegraphics[width=0.85\linewidth]{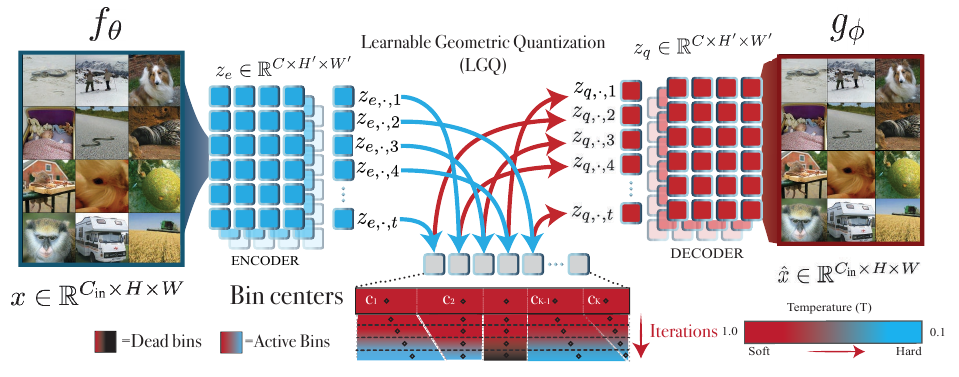}
    \caption{\textbf{LGQ tokenizer.} The encoder $f_\theta$ maps an image to continuous
    latents $z_e$; LGQ assigns each token to learnable codes via the annealed
    soft-assignment ($\tau{:}1.0\!\to\!0.1$), and the decoder $g_\phi$ reconstructs from
    the quantized $z_q$. The same backbone is shared by all six quantizers in our
    comparison.}
    \label{fig:arch}
\end{figure}

\section{Method: Learnable Geometric Quantization}
\label{sec:method}

\subsection{Setup}
We instantiate the tokenizer of \cref{sec:background} (\cref{fig:arch}) with $T=256$
spatial tokens (for $256{\times}256$ images at the downsampling factor $f=16$) and the
embedding dimension $C=64$. The quantizer maps each token $\mathbf{z}_{b,t} \in
\mathbb{R}^C$ to a discrete code $k \in \{1,\dots,K\}$ and embedding
$\hat{\mathbf{z}}_{b,t}$.

\subsection{Soft-assignment with Learnable Codes}
LGQ maintains $K$ learnable codes $\mathbf{e}_k \in \mathbb{R}^C$ and computes a
soft-assignment using squared distance $\ell_2$ and temperature $\tau$:
\begin{equation}
    p_{b,t,k} \;=\;
    \frac{\exp\!\left(-\|\mathbf{z}_{b,t} - \mathbf{e}_k\|_2^2 / \tau\right)}
         {\sum_{k'} \exp\!\left(-\|\mathbf{z}_{b,t} - \mathbf{e}_{k'}\|_2^2 / \tau\right)}.
    \label{eq:softmax}
\end{equation}
The decoded vector uses straight-through gradients~\cite{ste}:
\begin{equation}
    \hat{\mathbf{z}}_{b,t} \;=\; \mathbf{e}_{k^*} \;+\;
        \underbrace{\sum_k p_{b,t,k} \mathbf{e}_k - \sg\!\left[\sum_k p_{b,t,k} \mathbf{e}_k\right]}_{\text{soft path for gradient}},
\end{equation}
where $k^* = \arg\max_k p_{b,t,k}$ and $\sg[\cdot]$ is the stop-gradient. The forward pass
therefore commits to a single discrete code (as in standard VQ), while the backward pass
flows through the full soft-assignment, giving the codebook geometry a well-defined
gradient signal at every step. The score is the squared distance
$\|\mathbf{z}-\mathbf{e}_k\|_2^2$ rather than the unnormalized inner product
$\langle\mathbf{z},\mathbf{e}_k\rangle$ used by IBQ~\cite{ibq}; \cref{sec:vs-ibq} sets the
two rules against each other. Squared $\ell_2$ is also numerically convenient under the
annealing schedule below: at small $\tau$ the exponent saturates the softmax to a near-hard
one-hot without a separate safeguard.

\subsection{Temperature Annealing}
We linearly anneal $\tau$ from $\tau_{\text{start}}=1.0$ to $\tau_{\text{end}}=0.1$ over
training. Early high temperature gives near-uniform assignments and an exploration phase
in which codes can move freely and respond to the encoder distribution; late low
temperature commits each token to one code and lets the codebook geometry crystallize. We
ablate both endpoints and the schedule shape in \cref{subsec:abl_tau}.

\subsection{Regularization}
\label{sec:reg}
Let $\bar{p}_k = \tfrac{1}{BT}\sum_{b,t} p_{b,t,k}$ be the average usage of code $k$ over
the $B{\cdot}T$ tokens of one batch. Under data-parallel training $\bar p$ is formed
independently on each device from that device's local batch, with no cross-rank reduction,
so $B$ denotes the per-GPU batch size throughout. We add two terms to the reconstruction
loss:
\begin{align}
    \mathcal{L}_{\text{peak}}  &= \tfrac{1}{BT}\sum_{b,t} \ReLU\!\left(1 - \sum_k p_{b,t,k}^2\right), \\
    \mathcal{L}_{\text{div}}   &= K \sum_k \bar{p}_k^2.
\end{align}
$\mathcal{L}_{\text{peak}}$ encourages each token's soft-assignment distribution to
concentrate on a single code so that the soft-assignment behaves like a near-hard one-hot
at training time (\emph{peakedness}); $\mathcal{L}_{\text{div}}$ penalizes concentration
of average usage in the batch. The two terms are deliberately push-pull: peakedness asks
individual tokens to commit to a single code, while diversity asks the batch as a whole to
spread its commitments across the codebook. The $K$-scaled diversity term is the
load-bearing anti-collapse pressure while peakedness is a milder refinement that improves
per-token commitment for a further utilization and rFID gain. The regularizer-weight
ablation in \cref{subsec:abl_reg} isolates both effects.

\paragraph{Why $\mathcal{L}_{\text{div}}$ carries a factor of $K$} The $K$ scaling
normalizes the term against the uniform usage vector: if all $\bar p_k = 1/K$ then
$K\sum_k \bar p_k^2 = 1$ for any $K$. This normalization is exact while assignments remain
soft and $\bar p$ retains full support. Once $\mathcal{L}_{\text{peak}}$ and the annealed
temperature drive assignments toward one-hot, $\bar p$ is an average of $B{\cdot}T$
indicators and cannot spread over more than $B{\cdot}T$ codes, so the smallest value the
term can take grows with $K$; \cref{sec:theory} states the bound in both regimes. We use
one $\lambda_{\text{div}}$ across $K = 4{,}096$ to $65{,}536$, and \cref{tab:lgq_sweep}
reports the sweep it produces, where $\lambda_{\text{peak}}=\lambda_{\text{div}}=0.005$.

\paragraph{Total loss}
\begin{equation}
\mathcal{L} = \mathcal{L}_{\text{recon}} +
\lambda_{\text{peak}}\mathcal{L}_{\text{peak}}
   + \lambda_{\text{div}}\mathcal{L}_{\text{div}}.
\end{equation}
$\mathcal{L}_{\text{recon}}$ is the VQ-GAN reconstruction objective~\cite{vqgan}: a
mixture of a pixel-space $\ell_1$ term, an LPIPS perceptual term~\cite{lpips}, and a
PatchGAN hinge loss,
\begin{equation}
\mathcal{L}_{\text{recon}} =
   \lambda_{\ell_1}\lVert x - \hat{x}\rVert_1
 + \lambda_{\text{lpips}}\,\mathcal{L}_{\text{LPIPS}}(x,\hat{x})
 + \lambda_{\text{gan}}\,\mathcal{L}_{\text{GAN}}(\hat{x}),
\end{equation}
where $x$ is the input image, $\hat{x}$ the decoded reconstruction, and
$\mathcal{L}_{\text{GAN}}(\hat{x}) = -\mathbb{E}\!\left[D(\hat{x})\right]$ is the
generator term of the PatchGAN hinge loss with discriminator $D$ trained via
$\mathbb{E}\!\left[\max(0, 1 - D(x))\right] + \mathbb{E}\!\left[\max(0, 1 +
D(\hat{x}))\right]$. We set $\lambda_{\ell_1}{=}1.0$, $\lambda_{\text{lpips}}{=}1.0$, and
$\lambda_{\text{gan}}{=}0.1$; the discriminator is enabled after $10{,}000$ optimizer
steps. All six quantizers in our comparison are trained with this identical reconstruction
objective. The only change between runs is the quantizer.

\section{Theoretical Properties}
\label{sec:theory}

Although LGQ is presented as an empirical method, several properties of its assignment
rule admit clean proofs that motivate the regularizer and schedule choices in
\cref{sec:method}.

\subsection{Preliminaries: Annealed-Softmax Properties}
\label{sec:prelim}
LGQ's soft-assignment is the temperature-scaled softmax
$p^\star_k=\exp(-\|\mathbf{z}-\mathbf{e}_k\|_2^2/\tau)/Z$ (\cref{eq:softmax}), and
inherits three standard (non-LGQ-specific) properties of the Gibbs distribution that we
use below. \emph{(i)}~It is the unique simplex minimizer of the free energy
$\sum_k p_k\|\mathbf{z}-\mathbf{e}_k\|_2^2+\tau\sum_k p_k\log p_k$, so $\tau$ trades
distortion against entropy. \emph{(ii)}~As $\tau\!\to\!0^+$ it converges to the hard
nearest-neighbour indicator $\mathbb{I}[k{=}k^\star]$, so annealing
$\tau{:}\,1.0\!\to\!0.1$ interpolates from an exploratory near-uniform prior to committed
quantization. \emph{(iii)}~For $\|\mathbf{z}\|,\|\mathbf{e}_k\|\le M$ the assignment
Jacobian satisfies $\sum_k\|\nabla_{\mathbf{z}} p_k\|_2\le 4M/\tau$, so the soft path is
Lipschitz and the STE backward signal stays bounded across the schedule.

\subsection{Bridging VQ and FSQ}
By \cref{sec:prelim}, LGQ's assignment recovers VQ's hard nearest-neighbour rule as
$\tau\to 0$ yet stays smooth with a bounded Jacobian for any $\tau>0$, so the codebook
receives a non-degenerate gradient on all $K$ codes at every step where VQ's $\arg\min$
routes gradient through exactly one code per token. The $K$-scaled diversity term pushes
toward FSQ-style full utilization without freezing the codebook to a grid. LGQ inherits
VQ's adaptivity and FSQ's collapse-resistance, with the regularizer pair acting as the
bridge between the two regimes.

\subsection{Bounds on the Regularizers}
\label{sec:bounds}

\begin{proposition}[Peakedness lower bound]
\label{prop:peak}
For any $p\in\Delta^{K-1}$, $\sum_k p_k^2 \le 1$ with equality iff $p$ is one-hot. Hence
$\mathcal{L}_{\mathrm{peak}}=\mathbb{E}_n[\ReLU(1-\sum_k p_{n,k}^2)]\ge 0$, with equality
iff every token's assignment is one-hot.
\end{proposition}

\begin{proposition}[Lower bound on the diversity term]
\label{prop:div}
Let $S=\{k:\bar p_k>0\}$ be the support of $\bar p$ and let $m=|S|$ be its size. Then
\[
\mathcal{L}_{\mathrm{div}} \;=\; K\sum_k \bar p_k^2 \;\ge\; \frac{K}{m},
\]
with equality iff $\bar p$ is uniform on $S$. If every token's assignment is one-hot,
$\bar p$ is an average of $B{\cdot}T$ indicator vectors, so $m \le \min(B{\cdot}T,K)$ and
$\mathcal{L}_{\mathrm{div}} \ge K/\min(B{\cdot}T,K)$.
\end{proposition}

Both propositions follow from Cauchy--Schwarz; the proofs are given in \cref{app:proofs}.

Which value of $m$ applies depends on where training is. Early on the temperature is high,
each token's assignment is close to uniform, $\bar p$ has full support, and the bound reads
$\mathcal{L}_{\mathrm{div}} \ge 1$ for every $K$. Late in training
$\mathcal{L}_{\mathrm{peak}}$ and the annealed temperature commit each token to a single
code, and the bound becomes $K/(B{\cdot}T)$ once $B{\cdot}T$ is smaller than $K$, so the
floor the diversity term is pushed against grows with the codebook while
$\lambda_{\mathrm{div}}$ is held fixed. Replacing the $K$ prefactor with
$\min(B{\cdot}T,K)$ restores a floor of $1$ at every $K$; our runs use the $K$ prefactor,
and \cref{tab:lgq_sweep} reports what it delivers across the sweep.

The two bounds motivate a regularizer-weight ablation, reported in full in
\cref{app:ablations} and summarized here. These are reduced-budget runs ($50{,}000$ images,
$10$ epochs), so their absolute scores sit well below the full-budget results of
\cref{tab:k16k} and should be read as relative effects. Ablating the terms individually at
$K{=}16{,}384$: dropping peakedness leaves the model usable ($53.40$ to $56.47$ rFID,
utilization $47.7$ to $38.4\%$), whereas dropping diversity destroys it ($53.40$ to
$204.55$ rFID, utilization $47.7$ to $0.12\%$, \ie roughly $20$ of $16{,}384$ codes
alive). Removing both is no worse on rFID than removing diversity alone ($174.19$ against
$204.55$) and equally collapsed ($0.07\%$ against $0.12\%$ utilization). Diversity is
therefore the load-bearing anti-collapse pressure and peakedness a refinement on top of it,
an ordering rather than a pair of equals.

\section{Related Work}
\label{sec:related}

\subsection{Baseline Quantizers and Their Failure Modes}

We organize the related work around two questions: \emph{how does each method prevent
codebook collapse, and at what design cost?} \Cref{app:baselines} tabulates the answers
method by method.

Each baseline answers the collapse question by restricting something: where codes may sit,
how many of them receive gradient, or what the assignment is scored by. VQ~\cite{vqvae}
routes gradient through the $\arg\min$ to a single code per token, so unselected codes
drift into dead codes; EMA updates and resets treat the symptom.
RotVQ~\cite{rotation_trick} repairs the gradient path with a rotation and rescaling but
leaves the hard assignment in place. FSQ~\cite{fsq} and LFQ~\cite{lfq} remove collapse by
construction, FSQ by fixing an axis-aligned scalar grid and LFQ by restricting codes to the
corners of a sign-pattern hypercube, and both give up the ability to place capacity where
the data is dense. SimVQ~\cite{simvq} and VQGAN-LC~\cite{vqganlc} keep a codebook but tie
it to a shared learned transform or to frozen pretrained features, so individual codes
cannot migrate. IBQ~\cite{ibq} is closest to LGQ: it applies STE to a softmax over all $K$
codes, so every code receives a gradient, but scores them by the unnormalized dot product
$\mathbf{z}^\top\mathbf{e}_k$ under a fixed temperature with an entropy penalty on the
assignment distribution, where LGQ uses squared $\ell_2$, an annealed $\tau$, and
regularizers on the soft assignment. Per-method detail is given in \cref{app:baselines}.

\paragraph{Positioning} LGQ is subject to none of the three restrictions above. Codes move
independently, all codes receive gradient, and the score is the squared Euclidean
distance, the same quantity the reconstruction loss minimizes. Utilization is controlled
by the regularizers of \cref{sec:theory} rather than by the codebook parameterization, and
one $(\lambda_{\mathrm{peak}}, \lambda_{\mathrm{div}})$ holds from $K{=}4{,}096$ to
$K{=}65{,}536$ without per-$K$ retuning, which a fixed or tied codebook does not require
and a hard-assignment learned codebook does not achieve.

\subsection{\texorpdfstring{$\ell_2^2$}{l2-squared} and Dot-Product Logits Comparison (LGQ vs.\ IBQ)}
\label{sec:vs-ibq}
The closest baseline to LGQ at the algorithmic level is IBQ~\cite{ibq}, which shares our
two central design choices: it applies STE to a softmax over all $K$ codes, so every code
receives a gradient at every step rather than only the selected one. The methods diverge
in a single but consequential place, which is the quantity the softmax scores. IBQ scores
codes by the unnormalized inner product $\langle\mathbf{z},\mathbf{e}_k\rangle$, whereas
LGQ scores them by the squared $\ell_2$ distance $-\|\mathbf{z}-\mathbf{e}_k\|^2$.

The distinction is clearest from the expansion of the squared distance:
\begin{equation}
\|\mathbf{z}-\mathbf{e}_k\|^2
  = \|\mathbf{z}\|^2
    \;-\; 2\langle\mathbf{z},\mathbf{e}_k\rangle
    \;+\; \|\mathbf{e}_k\|^2 .
\label{eq:mse-expand}
\end{equation}
The MSE reconstruction loss therefore decomposes into three terms: the encoder-output norm
$\|\mathbf{z}\|^2$, the code norm $\|\mathbf{e}_k\|^2$, and the inner product
$\langle\mathbf{z},\mathbf{e}_k\rangle$ that couples them. For a fixed token, the first
term $\|\mathbf{z}\|^2$ is a constant shared by every code, so it has no effect on which
code minimizes the distance; the assignment is decided entirely by the remaining two terms,
$-2\langle\mathbf{z},\mathbf{e}_k\rangle + \|\mathbf{e}_k\|^2$. LGQ's distance-based score
retains both, which is exactly the Bayes-optimal nearest-neighbor quantizer for the MSE
objective. IBQ's dot-product score keeps only the inner-product term and \emph{completely
eliminates $\|\mathbf{e}_k\|^2$}.

The consequence is a systematic bias. Because IBQ's loss contains no term that constrains
or compensates for the code norms, a code can win an assignment by sheer magnitude rather
than by its proximity to $\mathbf{z}$. Enlarging $\|\mathbf{e}_k\|$ inflates
$\langle\mathbf{z},\mathbf{e}_k\rangle$ and makes that code more likely to be selected,
even when a smaller-norm code lies closer in Euclidean terms. The dot-product rule thus
coincides with the optimal rule only in the degenerate case where all $\|\mathbf{e}_k\|$
are equal. Once the norms vary, it systematically prefers larger-norm codes, with
per-sample excess distortion bounded by the maximum norm gap
$\max_{j,k}(\|\mathbf{e}_j\|^2-\|\mathbf{e}_k\|^2)$. This is a structural property of the
selection rule, and the bias grows with the heterogeneity of the learned codebook.

We treat the realized effect empirically rather than as a proof. The IBQ--LGQ rFID gap
($16.21$ vs.\ $12.32$ at $K{=}16{,}384$) motivates the $\ell_2^2$ choice but does not by
itself isolate it, since published IBQ also differs in its two-sided (double) quantization
loss, backbone depth, and entropy penalty.

\paragraph{Isolating the logit} We rerun LGQ changing \emph{only} the quantity the softmax
scores, holding backbone, seed, schedule, optimizer and budget fixed (\cref{tab:logit}).
Reconstruction separates the two arms by $1.36$ rFID over a matched three-epoch window
($12.26$ against $13.62$), against an epoch-to-epoch scatter of $0.92$ rFID between them.
Code recruitment separates them by much more. The dot-product arm has assigned the same
$13{,}084$ codes throughout, with $3{,}300$ codes never assigned once and none recovered,
while the $\ell_2^2$ arm climbs to $100\%$ of the codebook and holds it. This is the norm
term of \cref{eq:mse-expand} made visible. The dot-product score is unbounded above in
$\|\mathbf{e}_k\|$, so a code raises its score by growing rather than by moving closer to
the data. A large code wins assignments it is not nearest to, and a code that stops being
selected receives no assignment gradient with which to grow back. The squared distance
admits no such shortcut, since its score is maximized only by a code sitting exactly on the
data point. The codebook norms confirm this. At the last audited checkpoint the $3{,}300$
dead codes have $\|\mathbf{e}_k\|\in[0.51,0.63]$ against $[5.20,5.85]$ for the $13{,}084$
live ones, two disjoint ranges with $\mathrm{corr}(\|\mathbf{e}_k\|,\log\bar{p}_k)=0.997$,
and every dead norm shrinks between consecutive checkpoints. Under
$-\|\mathbf{z}-\mathbf{e}_k\|^2$ the same correlation is $0.08$ and no code is dead
(\cref{app:vs-ibq}). The ablation isolates the score with LGQ's regularizers held fixed.
Published IBQ pairs the same dot-product score with a low-temperature entropy penalty on
the assignment distribution and reaches $99.99\%$ utilization (\cref{tab:k16k}). That
containment is size-specific: at $K{=}65{,}536$ under the identical protocol, IBQ's
utilization peaks at $98.2\%$ and falls to $55.1\%$ while LGQ's rises to $100\%$
(\cref{app:vs-ibq}). The misallocation the score permits re-emerges as the codebook grows
(\cref{sec:results-sweep}).

\begin{table}[t]
\centering
\caption{\textbf{Matched logit ablation.} Identical backbone, seed, schedule, and compute;
only the quantity the softmax scores differs. rFID is the mean over a matched three-epoch
window, since a single epoch is within run-to-run noise. ``Active'' is the number of codes
assigned at least once on the validation set, out of $K{=}16{,}384$.}
\label{tab:logit}
\small
\begin{tabular}{@{}lccc@{}}
\toprule
\kilth{Softmax score} & \kilth{rFID $\downarrow$} & \kilth{Util.\ \% $\uparrow$} & \kilth{Active} \\
\midrule
$-\|\mathbf{z}-\mathbf{e}_k\|^2$ \;(LGQ)              & \best{12.26} & \best{100.00} & \best{16{,}384} \\
$\langle\mathbf{z},\mathbf{e}_k\rangle$ \;(IBQ-style) & $13.62$      & $79.86$       & $13{,}084$ \\
\bottomrule
\end{tabular}
\end{table}

The two propositions and their proofs, together with a probabilistic (Gaussian mixture
vs.\ von Mises--Fisher) reading of the two selection rules and the gradient/compute
trade-off they imply, are given in \cref{app:vs-ibq}.

\paragraph{Relation to other codebook designs and regularizers} Residual and multi-stage
quantizers such as RQ-VAE~\cite{rqvae} address large-$K$ capacity by composing several
small codebooks; LGQ is orthogonal to this axis as a single-stage learnable codebook, and
the two are combinable. Nor are our regularizers new in isolation: usage and entropy
penalties, temperature-annealed soft-assignment, and geometric spreading terms all appear
in prior VQ variants (\cref{app:baselines}). The contribution is the specific combination,
$\ell_2^2$ soft-assignment with annealing and the push-pull peakedness and $K$-scaled
diversity pair, which removes collapse without EMA, resets, or reparameterization and
holds across the complete $K$ sweep under one setting.

\section{Experimental Setup}
\label{sec:setup}

\subsection{Backbone}
All quantizers share an identical VQ-GAN-style CNN autoencoder backbone~\cite{vqgan} with
base channel width 128 (widening to 256 at the deepest stage) and embedding dimension
$C=64$, downsampling factor $f=16$, producing $T=256$ spatial tokens for $256{\times}256$
images. The encoder/decoder include residual blocks and self-attention at the
$16{\times}16$ feature map. The replacement of the quantizer is the only change between
runs in \cref{tab:k16k}. The complete pipeline is diagrammed in \cref{fig:arch}. Training
uses DDP across $4$ GPUs at per-GPU batch size $8$, so each rank holds $B{\cdot}T{=}2048$
tokens per batch; the full optimizer and data settings are in \cref{app:training}.

\subsection{Evaluation Protocol}
\paragraph{Reconstruction} We report PSNR, SSIM~\cite{ssim}, LPIPS~\cite{lpips},
reconstruction loss and codebook utilization (fraction of entries used at least once on
the validation set). We also report \textbf{rFID}~\cite{fid}: Fr\'echet Inception Distance
between encoder--decoder reconstructions of the validation set and a matched real-image
reference set, computed against a frozen reference. Every rFID we report is computed on a
fixed $10{,}000$-image subset of the ImageNet validation split, identical for every row, so
the values are comparable within this paper but are not comparable to published rFID
computed on all $50{,}000$ images. Recomputing on the full split at each run's best
checkpoint leaves the ordering unchanged (LGQ $6.05$, SimVQ $8.19$, IBQ $8.42$, FSQ
$8.83$, RotVQ $12.76$).

\paragraph{Training runs and seeds} Each tokenizer is trained once per setting: every row
of \cref{tab:k16k,tab:lgq_sweep,tab:geometry,tab:logit,tab:maskgit_gen} comes from a single
training run. The $\pm$ values in \cref{tab:maskgit_gen} are computed over four sampling
seeds at a fixed trained tokenizer and quantify sampling variability alone. Per-epoch
reconstruction curves for every method, from which the epoch-to-epoch scatter of a single
run can be read directly, are given in \cref{app:curves}.

\paragraph{Best-epoch reporting under a matched compute budget} All methods are trained
under an identical wall-time budget, with periodic checkpointing and the same validation
cadence. For each method, we report the best-epoch checkpoint reached within the budget,
the standard early-stopping protocol, so that within-paper comparisons reflect each
method's best quality under matched compute rather than an arbitrary cutoff.

\subsection{Downstream MaskGIT Stage}
\label{subsec:maskgit_setup}
The choice of MaskGIT, rather than the larger LlamaGen-style autoregressive transformers
used by IBQ~\cite{ibq} or the LM-style decoder of MAGVIT-v2~\cite{lfq}, is deliberate: at
our scale it isolates the tokenizer's contribution without conflating it with a
multi-billion-parameter decoder. After tokenizer training, each tokenizer is frozen and
used to convert images to discrete token sequences. On top of the frozen tokenizer we
train a $50.8$M-parameter MaskGIT-style bidirectional transformer~\cite{maskgit} ($12$
layers, hidden size $512$) to predict randomly masked visual tokens, using a cosine masking
schedule, a start-of-sequence token, and a separate mask-token ID. We train on the full
ImageNet training set ($1{,}281{,}167$ images, $1{,}000$ classes) at $256{\times}256$ with
codebook size $K=16{,}384$, an effective batch of $128$ (per-step batch $32$ with gradient
accumulation $4$), AdamW at $\mathrm{lr}=10^{-4}$, and early stopping on the validation
NLL with patience $8$--$10$. All tokenizers are trained under matched settings, and the
leading tokenizers reach ${\sim}100$ epochs, by which point validation NLL has plateaued.

For the generation-side comparison (\cref{sec:results-maskgit}, \cref{tab:maskgit_gen}) we
follow MAGVIT-v2~\cite{lfq} and IBQ~\cite{ibq} in reporting gFID, and Inception Score
where available, on $50$K samples drawn from the trained sampler under the standard cosine
schedule with $11$ decoding steps.

\section{Results}
\label{sec:results}

\subsection{Reconstruction at \texorpdfstring{$K=16{,}384$}{K=16,384}}
\label{sec:results-main}

\Cref{tab:k16k} reports the validation-set reconstruction metrics rFID, PSNR, SSIM, LPIPS,
reconstruction loss, codebook utilization, and perplexity for each method. All six runs
share the identical $\ell_1$\,+\,LPIPS\,+\,PatchGAN reconstruction objective described in
\cref{sec:method}. The only change between runs is the quantizer. LGQ achieves the best
PSNR, SSIM, LPIPS, and rFID across the group; \cref{fig:recon} shows the corresponding
reconstructions. \Cref{tab:k16k} carries no plain-VQ row: RotVQ serves as the
learnable-argmin baseline, and VQ is retained in \cref{app:baselines} as a design
reference.

\begin{table}[t]
\centering
\caption{\textbf{Reconstruction at $K=16{,}384$, ImageNet $256{\times}256$, identical
backbone.} Best-epoch validation metrics from each run is reported. All six methods are
trained under the same matched compute budget; we report the lowest-rFID checkpoint
reached for each. \best{Bold} marks the best in each metric column.}
\label{tab:k16k}
\small
\setlength{\tabcolsep}{5pt}
\begin{tabular}{@{}lccccccc@{}}
\toprule
\kilth{Method} & \kilth{PSNR $\uparrow$} & \kilth{SSIM $\uparrow$} & \kilth{LPIPS $\downarrow$} & \kilth{Rec.\ loss $\downarrow$} & \kilth{rFID $\downarrow$} & \kilth{Util.\ \%} & \kilth{Perplexity} \\
\midrule
RotVQ     & 20.572 & 0.5788 & 0.3032 & 0.2635 & 20.15 & 89.83  & 1{,}921 \\
LFQ       & 21.124 & 0.6013 & 0.2858 & 0.2450 & 18.82 & 100.00 & 16{,}046 \\
FSQ       & 21.475 & 0.6014 & 0.2705 & 0.2322 & 17.38 & 100.00 & 13{,}651 \\
IBQ       & 21.339 & 0.6145 & 0.2661 & 0.2373 & 16.21 & 99.99  & 15{,}101 \\
SimVQ     & 20.930 & 0.5990 & 0.2778 & 0.2512 & 16.05 & 100.00 & 3{,}981 \\
\textbf{LGQ (ours)} & \best{21.791} & \best{0.6289} & \best{0.2438} & \best{0.2225} & \best{12.32} & 100.00 & \best{16{,}011} \\
\bottomrule
\end{tabular}
\end{table}

\begin{figure}[t]
\centering
\includegraphics[width=0.98\linewidth]{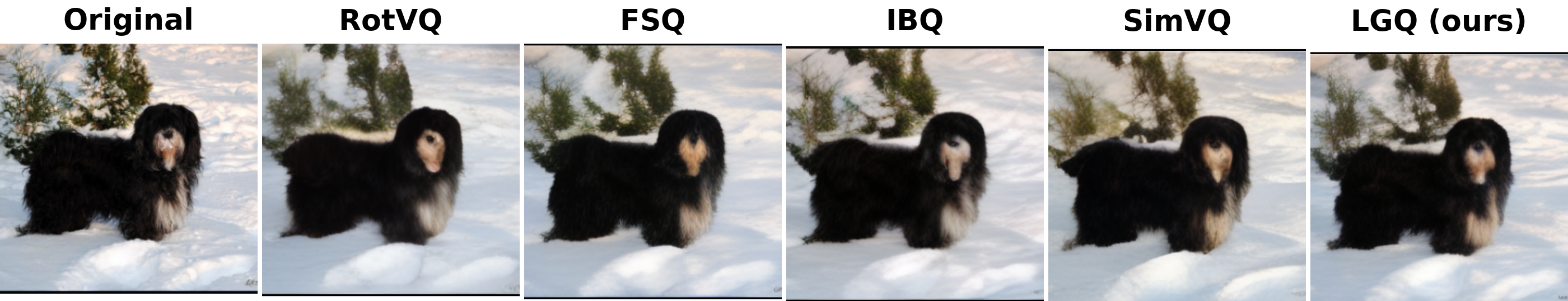}
\caption{\textbf{Reconstructions at $K{=}16{,}384$.} One ImageNet validation image, then
its reconstruction from five of the six frozen tokenizers. LGQ retains the finest texture
and color fidelity, consistent with its best rFID/LPIPS in \cref{tab:k16k}. The full
per-method grid over six images is in \cref{app:qualitative}.}
\label{fig:recon}
\end{figure}

\subsection{LGQ Codebook-Size Sweep}
\label{sec:results-sweep}

\Cref{tab:lgq_sweep} reports LGQ across five codebook sizes. PSNR, SSIM, and LPIPS improve
with $K$, and rFID drops monotonically from $16.91$ at $K=4{,}096$ to $10.72$ at
$K=65{,}536$. The key result is that utilization reaches and holds $100\%$ at every $K$,
so the learnable codebook never collapses as it grows $16\times$, under a single fixed set
of regularizer weights with no per-$K$ retuning. What changes with $K$ is how long
saturation takes: the first epoch at which utilization reaches $100\%$ is epoch~$7$ at
$K{=}4{,}096$, then $10$, $13$, $22$ and $29$ as $K$ grows to $65{,}536$, and once reached
it is held at every later epoch we measure. Each row is read at or after its own
saturation epoch, so a larger codebook is given the extra epochs it needs rather than a
checkpoint chosen to flatter it. All four reconstruction metrics continue to improve
through $K=65{,}536$ with no sign of levelling off at this backbone capacity. To test
whether this behaviour is specific to LGQ, \cref{app:vs-ibq} trains IBQ at $K{=}4{,}096$,
$16{,}384$ and $65{,}536$ under the identical protocol: LGQ is ahead on every metric at all
three sizes at a matched epoch, and where IBQ's utilization falls as $K$ grows, LGQ's rises
to $100\%$ (\cref{sec:vs-ibq}).

\begin{table}[t]
\centering
\caption{\textbf{LGQ codebook-size sweep on ImageNet $256{\times}256$.} For each $K$ we
report the converged checkpoint at which codebook utilization saturates to $100\%$.
\best{Bold} marks the best in each metric column. A single set of regularizer weights
($\lambda_{\text{peak}}{=}\lambda_{\text{div}}{=}0.005$) and a single temperature schedule
($1.0\!\to\!0.1$ linear) is used across the entire sweep without per-$K$ retuning.}
\label{tab:lgq_sweep}
\small
\setlength{\tabcolsep}{5pt}
\begin{tabular}{@{}rccccccc@{}}
\toprule
\kilth{$K$} & \kilth{PSNR $\uparrow$} & \kilth{SSIM $\uparrow$} & \kilth{LPIPS $\downarrow$} & \kilth{Rec.\ loss $\downarrow$} & \kilth{rFID $\downarrow$} & \kilth{Util.\ \%} & \kilth{Perplexity} \\
\midrule
$4{,}096$  & 21.353 & 0.6080 & 0.2751 & 0.2358 & 16.91 & 100.00 & 3{,}997 \\
$8{,}192$  & 21.494 & 0.6167 & 0.2610 & 0.2312 & 14.59 & 100.00 & 7{,}893 \\
$16{,}384$ & 21.791 & 0.6289 & 0.2438 & 0.2225 & 12.32 & 100.00 & 16{,}011 \\
$32{,}768$ & 21.932 & 0.6386 & 0.2371 & 0.2180 & 11.90 & 100.00 & 31{,}826 \\
$65{,}536$ & \best{22.183} & \best{0.6491} & \best{0.2307} & \best{0.2123} & \best{10.72} & 100.00 & \best{63{,}077} \\
\bottomrule
\end{tabular}
\end{table}

\subsection{Codebook Geometry Analysis}
\label{sec:results-geometry}

To move beyond aggregate rFID and characterize \emph{how} each codebook distributes
information, we encode $5{,}000$ ImageNet validation images through each tokenizer and
compute four statistics over the resulting $(N,T)$ index tensors ($T{=}256$ spatial tokens
per image, $K{=}16{,}384$): normalized \textbf{marginal entropy}
$H_{\text{norm}}{=}H(p)/\log K$ (usage uniformity across tokens: $1$ is perfectly uniform,
low values flag dominant ``magnet codes''); \textbf{top-1\% mass} (the share of
probability captured by the most-used $1\%$ of codes, low is balanced);
\textbf{co-occurrence density} (the fraction of code-pairs that ever appear within the
same image; high indicates compositional, freely-combining codes); and \textbf{pairwise
mutual information} between two random spatial positions (high indicates structured,
non-redundant spatial coding).

\begin{table}[t]
\centering
\caption{\textbf{Codebook geometry at $K{=}16{,}384$.} LGQ leads on all four metrics; IBQ
is second on each. RotVQ is severely concentrated, with $43\%$ of usage in its top $1\%$
of codes. Utilization here is on the $5{,}000$-image analysis subset, so RotVQ reads lower
than its full-validation value in \cref{tab:k16k}. MI is a plug-in estimate over $32$
random spatial position pairs from the same $5{,}000$ images. That estimator saturates near
$\log N{=}8.52$ nats, close to the values reported here, so the MI column should be read as
a ranking and not as a measurement of absolute mutual information.}
\label{tab:geometry}
\small
\begin{tabular}{@{}lccccc@{}}
\toprule
\kilth{Tokenizer} & \kilth{Util.\,\%} & \kilth{$H_{\text{norm}}$ $\uparrow$} & \kilth{Top-1\% $\downarrow$} & \kilth{Co-occ.\,\% $\uparrow$} & \kilth{MI $\uparrow$} \\
\midrule
\textbf{LGQ (ours)} & \best{100.0} & \best{0.997} & \best{1.9} & \best{13.1} & \best{8.08} \\
IBQ       & 100.0 & 0.991 &  2.6 & 12.9 & 8.01 \\
FSQ       & 100.0 & 0.977 &  3.7 & 12.2 & 7.94 \\
SimVQ     & 100.0 & 0.854 & 28.9 &  6.2 & 6.34 \\
RotVQ     &  79.9 & 0.771 & 43.4 &  4.0 & 5.20 \\
\bottomrule
\end{tabular}
\end{table}

\Cref{tab:geometry} shows that LGQ's codebook is near-perfectly uniform
($H_{\text{norm}}{=}0.997$), so no small subset of codes absorbs the token distribution.
Its top-$1\%$ mass of only $1.9\%$ confirms that no ``dead star'' codes dominate; in
contrast, RotVQ allocates $43\%$ of total usage to its top $164$ codes: a ${\sim}23\times$
concentration relative to LGQ.

The co-occurrence density ($13.1\%$) and pairwise MI ($8.08$ nats) are also highest for
LGQ, indicating codes that combine richly across spatial positions; IBQ is the closest
competitor (co-occurrence $12.9\%$, MI $8.01$), with FSQ just behind. While FSQ and SimVQ
also achieve 100\% utilization, their lower mutual information and co-occurrence scores
indicate that global utilization alone does not guarantee a rich latent space. A UMAP
overlay of the learned codes on the continuous encoder distribution, given in
\cref{app:geometry}, shows the same thing: the baselines group into isolated subregions
where LGQ spreads smoothly across the distribution.

Codebook uniformity is positively but imperfectly associated with reconstruction quality
(\cref{tab:k16k}): the ordering is anchored by LGQ at both extremes but is not monotone,
SimVQ reaches a strong rFID ($16.05$) at only moderate uniformity
($H_{\text{norm}}{=}0.854$), so uniformity is best read as one contributing factor, which
LGQ's regularizers pursue as a mechanism to keep the full codebook available to the
decoder rather than as an end in itself.

\subsection{Downstream Generation with MaskGIT}
\label{sec:results-maskgit}
Using the MaskGIT transformers trained on each frozen tokenizer, we report the
class-conditional $50$K-sample generation comparison at $K=16{,}384$ in
\cref{tab:maskgit_gen}; per-tokenizer sample grids are in \cref{app:qualitative}. LGQ
attains the best gFID ($57.69$), extending its reconstruction lead into generation, with
IBQ second ($59.47$); IBQ attains the best IS ($15.03$). Ordered by gFID, LGQ produces the
sharpest and most class-coherent samples and RotVQ ($92.99$ gFID) the blurriest.

\paragraph{Utilization is not sufficient} FSQ, SimVQ, and LGQ all reach $100.00\%$
utilization at $K{=}16{,}384$ (\cref{tab:k16k}) yet span $12.5$ gFID ($70.16$, $69.38$,
$57.69$), so a fully used codebook does not by itself determine generation quality: high
utilization shows that codes are assigned, not that quantization error is low or that the
resulting token distribution is easy to model. Our claim is that LGQ attains high
utilization, the best reconstruction, and the best generation \emph{jointly}. The finer
usage statistics of \cref{sec:results-geometry} do separate the fully-utilizing methods
and order them consistently with quality, but they too are measured alongside quality
rather than intervened upon; a causal test would require intervening on the token
distribution with reconstruction held fixed, which we do not attempt.

\begin{table}[t]
\centering
\caption{\textbf{Class-conditional MaskGIT generation at $K{=}16{,}384$.} gFID and IS are
measured on $50$K class-conditionally sampled images ($T{=}11$ steps, identical sampling
settings across tokenizers) against ImageNet-256 reference statistics, following
MAGVIT-v2~\cite{lfq} and IBQ~\cite{ibq}. Each tokenizer is a single training run; for LGQ,
IBQ, and FSQ we report mean\,$\pm$\,std over four \emph{sampling} seeds at that fixed
tokenizer. Each MaskGIT is trained on its method's \emph{best-rFID} tokenizer and evaluated
at its best-val-NLL checkpoint within a matched wall-clock budget.}
\label{tab:maskgit_gen}
\small
\begin{tabular}{@{}lccc@{}}
\toprule
\kilth{Tokenizer} & \kilth{rFID $\downarrow$} & \kilth{gFID $\downarrow$} & \kilth{IS $\uparrow$} \\
\midrule
\textbf{LGQ (ours)} & \best{12.32} & \best{57.69\,{\footnotesize$\pm$0.26}} & 14.80 \\
IBQ        & 16.21 & 59.47\,{\footnotesize$\pm$0.09} & \best{15.03} \\
SimVQ      & 16.05 & 69.38 & 13.26 \\
FSQ        & 17.38 & 70.16\,{\footnotesize$\pm$0.24} & 13.60 \\
RotVQ      & 20.15 & 92.99 & 10.63 \\
\bottomrule
\end{tabular}
\end{table}

\subsection{Computational Cost}
\label{subsec:cost}

All-code soft assignment is widely assumed to be expensive; the full timing and memory
table is in \cref{app:cost}. Timed in isolation at the training configuration, the
quantizer accounts for $\sim$$10\%$ of a full optimizer step at $K{=}16{,}384$, where LGQ
is $3.8\times$ faster and $1.5\times$ smaller than IBQ; IBQ exhausts memory at
$K{=}262{,}144$, where LGQ still trains. FSQ is cheaper than both and $K$-independent, so
our claim is best reconstruction per unit cost, not lowest cost. End-to-end the picture is
flatter still: one epoch costs $15$--$17$ GPU-hours for \emph{every} quantizer at
$K{=}16{,}384$, so the quantizer is not the end-to-end bottleneck at any $K$ we train at.

\section{Discussion}
\label{sec:discussion}
Under a fixed VQ-GAN-style objective with all quantizers sharing an identical backbone and
budget, LGQ achieves the best rFID, a $-3.73$ gap to the next baseline, from a single
untuned configuration; in the ablations of \cref{app:ablations} rFID stays within a narrow
band along the symmetric diagonal from $\lambda_{\text{peak}}{=}\lambda_{\text{div}}{=}0.001$
to $0.05$ and across annealing-schedule shapes. The advantage holds as the codebook grows
$16\times$ to $K{=}65{,}536$ with utilization saturated throughout, and extends to
downstream generation, where LGQ attains the best gFID. Notably, the baseline ordering
does not transfer cleanly from reconstruction to generation, echoing prior reports that
stronger reconstruction need not yield better generation.

\paragraph{Limitations and future work}
This study has three principal limitations. First, our evaluation is confined to a single
modality: all experiments are conducted on static images at ImageNet $256{\times}256$, and
we do not assess whether LGQ's advantages extend to other modalities such as audio or
video. Second, the codebook-size sweep covers LGQ at five values of $K$ and one baseline,
IBQ, at $K{=}4{,}096$, $16{,}384$ and $65{,}536$; the remaining quantizers are evaluated
only at the matched size $K{=}16{,}384$. We show that LGQ's $K$-scaled diversity
regularizer stays calibrated across the sweep without re-tuning and that IBQ loses
utilization over the same range, but we do not characterize how the remaining quantizers
scale; extending the sweep to every quantizer is left to future work. Third, each tokenizer
is trained once per setting, so we report no variance across training seeds; the intervals
in \cref{tab:maskgit_gen} are over sampling seeds at a fixed tokenizer, and the per-epoch
curves in \cref{app:curves} are the only evidence we offer on run-level stability. We
further note that, because our shared backbone and matched training budget differ from
those of the original baseline papers, we regard the \emph{relative} ranking obtained
under this identical-backbone protocol as the claim of this work, rather than the absolute
scores.

\section{Conclusion}
\label{sec:conclusion}
We introduced Learnable Geometric Quantization (LGQ), a vector quantizer that eliminates
codebook collapse without EMA or reset heuristics. By combining soft-to-hard temperature
annealing with lightweight peakedness and $K$-scaled diversity regularizers, LGQ scales to
large codebooks while maintaining $100\%$ utilization under one hyperparameter setting. On
ImageNet $256{\times}256$ within a fixed VQ-GAN backbone, it achieves the best rFID, PSNR,
SSIM, and LPIPS at $K{=}16{,}384$, and keeps improving as the codebook grows to
$K{=}65{,}536$. Trained on the frozen tokenizers, MaskGIT transformers achieve the best
class-conditional gFID with LGQ, leading on reconstruction and generation alike.



\begin{availability}
Code and training configurations are available at
\url{https://anonymous.4open.science/r/lgq-anon-E12C/}.
ImageNet ILSVRC-2012~\cite{imagenet} is public and cited in \cref{sec:setup}.
\end{availability}

\begin{conflicts}
The authors declare no competing interests.
\end{conflicts}

\bibliography{references}

\appendix

\section{Ablations}
\label{app:ablations}
We isolate the two load-bearing design choices in LGQ: the regularizer weights
($\lambda_{\text{peak}}$, $\lambda_{\text{div}}$) and the temperature schedule
($\tau_{\text{start}}$, $\tau_{\text{end}}$, schedule shape). All ablation runs use the
same backbone and reconstruction objective as the main comparison (\cref{sec:setup}),
$K=16{,}384$, trained on a reduced $50{,}000$-image / $10$-epoch budget so the full grid is
tractable. Absolute scores are therefore lower than the full-budget results of
\cref{sec:results} and should be read as \emph{relative} effects; row~1 of each table is
the default setting at this same reduced budget.

\subsection{Regularizer-Weight Ablation}
\label{subsec:abl_reg}
\Cref{tab:abl_reg} sweeps $\lambda_{\text{peak}}$ and $\lambda_{\text{div}}$ jointly around
the default $0.005$ setting, and also disables each regularizer in isolation. The two
``$\lambda{=}0$'' rows test the necessity of each term: the push-pull formulation of
\cref{sec:reg} predicts that peakedness alone collapses (peaked tokens, but only a few
codes used across the batch) and that diversity alone leaves soft assignments diffuse
(uniform usage, but no commitment to a single code at the token level). The 2D grid around
$0.005$ tests whether the chosen weight sits in a wide flat region or on a knife edge.
Both predictions are borne out. Removing the diversity term ($\lambda_{\text{div}}{=}0$)
collapses the codebook to $0.12\%$ utilization (rFID $204.55$), and disabling both terms is
no better: it is equally collapsed ($0.07\%$ utilization) at a marginally lower rFID
($174.19$); even merely weakening diversity to $0.001$ while keeping full peakedness
already drops usage to $5.32\%$. Dropping peakedness alone ($\lambda_{\text{peak}}{=}0$) is
far less damaging: utilization stays at $38.35\%$ and rFID degrades only mildly
($53.40\!\to\!56.47$) which is consistent with the diversity term, not peakedness, being
the load-bearing anti-collapse pressure. Across the symmetric grid from $0.001$ to $0.05$
rFID stays within a narrow $50.5$--$53.4$ band, indicating the default $0.005$ sits in a
wide flat region rather than on an edge.

\begin{table}[t]
\centering
\caption{\textbf{Regularizer-weight ablation.} LGQ at $K=16{,}384$, varying
$\lambda_{\text{peak}}$ and $\lambda_{\text{div}}$ jointly. Row~1 is the default setting;
the ``$\lambda_{\text{peak}}{=}0$'' and ``$\lambda_{\text{div}}{=}0$'' rows test the
necessity of each regularizer in isolation. Reduced-budget runs (see
\cref{app:ablations}); scores are relative, not comparable to the full-budget main
results.}
\label{tab:abl_reg}
\small
\setlength{\tabcolsep}{4pt}
\begin{tabular}{@{}cccccccl@{}}
\toprule
\kilth{$\lambda_{\text{peak}}$} & \kilth{$\lambda_{\text{div}}$} & \kilth{PSNR $\uparrow$} & \kilth{SSIM $\uparrow$} & \kilth{LPIPS $\downarrow$} & \kilth{rFID $\downarrow$} & \kilth{Util.\ \%} & \kilth{Notes} \\
\midrule
$0.005$  & $0.005$  & 20.835 & 0.5833 & 0.3880 & 53.40  & 47.66 & default \\
\midrule
$0$      & $0.005$  & 20.679 & 0.5626 & 0.3917 & 56.47  & 38.35 & no peak: usable, weaker commitment \\
$0.005$  & $0$      & 17.583 & 0.4369 & 0.5601 & 204.55 & 0.12  & no div: codebook collapse \\
$0$      & $0$      & 17.695 & 0.4390 & 0.5453 & 174.19 & 0.07  & both off: collapse \\
\midrule
$0.001$  & $0.001$  & 20.796 & 0.5762 & 0.3718 & 50.48  & 21.61 & weak regularization \\
$0.001$  & $0.005$  & 20.586 & 0.5783 & 0.3788 & 49.74  & 40.30 & asymmetric: weak peak, std div \\
$0.005$  & $0.001$  & 20.304 & 0.5442 & 0.4224 & 75.59  & 5.32  & asymmetric: std peak, weak div \\
$0.01$   & $0.01$   & 20.947 & 0.5772 & 0.3636 & 51.29  & 48.85 & 2$\times$ default \\
$0.05$   & $0.05$   & 20.787 & 0.5782 & 0.3796 & 52.08  & 68.24 & 10$\times$ default (over-reg.\ test) \\
\bottomrule
\end{tabular}
\end{table}

\subsection{Temperature-Schedule Ablation}
\label{subsec:abl_tau}
\Cref{tab:abl_tau} ablates the temperature schedule. The default anneals $\tau$ linearly
from $1.0$ to $0.1$ over training. We vary the starting temperature ($0.5$, $1.0$, $2.0$),
the ending temperature ($0.05$, $0.1$, $0.3$), the constant-$\tau$ baseline (no
annealing), and the schedule shape (linear vs.\ cosine). Two failure modes seemed likely a
priori: starting or ending too hot should leave the soft assignment diffuse late in
training, and starting or ending too cold should suppress early exploration, lowering
utilization and raising rFID. The cold-start mode appears sharply. A constant cold schedule
($\tau{=}0.1$) drops utilization to $4.33\%$ at rFID $113.80$, and a cooler start ($0.5 \to
0.1$) already cuts usage to $8.37\%$. The hot-start mode, by contrast, does not appear at
this budget: the hottest start ($2.0 \to 0.1$) is the strongest row in the table (rFID
$36.00$, utilization $77.18\%$), ahead of the default $1.0 \to 0.1$ ($53.40$, $47.66\%$).
For $\tau_{\text{start}} \ge 1.0$ the ending temperature, the cosine shape, and the
constant-$1.0$ baseline all fall within noise of the default ($52.8$ to $54.1$ rFID). The
schedule is therefore forgiving on the warm side and brittle only when started cold. One
consequence matters for how the main results should be read. The $1.0 \to 0.1$ schedule
used for every main-table LGQ run is not the best setting in this ablation; a hotter start
($2.0 \to 0.1$) is better here. We hold a single $1.0 \to 0.1$ schedule across the entire
study rather than tuning per setting, so the headline LGQ numbers, which are already the
best among all tokenizers, are reported at a conservative, untuned operating point. This
ablation is run at reduced budget; whether the hot-start advantage persists at full budget
is left to future work.

\begin{table}[t]
\centering
\caption{\textbf{Temperature-schedule ablation.} LGQ at $K=16{,}384$, varying
$\tau_{\text{start}}$, $\tau_{\text{end}}$, and the annealing shape. ``Schedule'' column:
\textsc{lin} = linear, \textsc{cos} = cosine, \textsc{const} = no annealing. Row~1 is the
default setting. Reduced-budget runs (see \cref{app:ablations}).}
\label{tab:abl_tau}
\small
\setlength{\tabcolsep}{5pt}
\begin{tabular}{@{}cccccccc@{}}
\toprule
\kilth{$\tau_{\text{start}}$} & \kilth{$\tau_{\text{end}}$} & \kilth{Schedule} & \kilth{PSNR $\uparrow$} & \kilth{SSIM $\uparrow$} & \kilth{LPIPS $\downarrow$} & \kilth{rFID $\downarrow$} & \kilth{Util.\ \%} \\
\midrule
$1.0$  & $0.1$  & \textsc{lin}   & 20.835 & 0.5833 & 0.3880 & 53.40  & 47.66 \\
\midrule
$0.5$  & $0.1$  & \textsc{lin}   & 20.446 & 0.5548 & 0.4011 & 62.95  & 8.37 \\
$2.0$  & $0.1$  & \textsc{lin}   & 20.766 & 0.5807 & 0.3412 & 36.00  & 77.18 \\
$1.0$  & $0.05$ & \textsc{lin}   & 20.589 & 0.5622 & 0.3887 & 53.67  & 47.27 \\
$1.0$  & $0.3$  & \textsc{lin}   & 20.558 & 0.5804 & 0.3860 & 52.91  & 48.55 \\
\midrule
$1.0$  & $0.1$  & \textsc{cos}   & 20.680 & 0.5850 & 0.3863 & 54.06  & 50.56 \\
$1.0$  & $1.0$  & \textsc{const} & 20.638 & 0.5718 & 0.3820 & 52.78  & 52.86 \\
$0.1$  & $0.1$  & \textsc{const} & 19.378 & 0.4993 & 0.4852 & 113.80 & 4.33 \\
\bottomrule
\end{tabular}
\end{table}

\FloatBarrier

\section{Proofs}
\label{app:proofs}
This section proves the two propositions of \cref{sec:bounds}, restated here so the
appendix is self-contained.

\begin{proposition*}[Restatement of \cref{prop:peak}]
For any $p\in\Delta^{K-1}$, $\sum_k p_k^2 \le 1$ with equality iff $p$ is one-hot. Hence
$\mathcal{L}_{\mathrm{peak}}=\mathbb{E}_n[\ReLU(1-\sum_k p_{n,k}^2)]\ge 0$, with equality
iff every token's assignment is one-hot.
\end{proposition*}

\begin{proof}[Proof of \cref{prop:peak}]
By Cauchy--Schwarz applied to $p_k\cdot p_k$,
$\sum_k p_k^2 \le (\max_k p_k)\sum_k p_k = \max_k p_k \le 1$, with both equalities iff some
$p_{k^\star}=1$.
\end{proof}

\begin{proposition*}[Restatement of \cref{prop:div}]
Let $S=\{k:\bar p_k>0\}$ be the support of $\bar p$ and let $m=|S|$ be its size. Then
$\mathcal{L}_{\mathrm{div}} = K\sum_k \bar p_k^2 \ge K/m$, with equality iff $\bar p$ is
uniform on $S$. If every token's assignment is one-hot, $\bar p$ is an average of
$B{\cdot}T$ indicator vectors, so $m \le \min(B{\cdot}T,K)$ and
$\mathcal{L}_{\mathrm{div}} \ge K/\min(B{\cdot}T,K)$.
\end{proposition*}

\begin{proof}[Proof of \cref{prop:div}]
Apply Cauchy--Schwarz on the support:
$1=\bigl(\sum_{k\in S}\bar p_k\cdot 1\bigr)^2 \le m\sum_{k\in S}\bar p_k^2
 = m\sum_k \bar p_k^2$, so $\sum_k \bar p_k^2 \ge 1/m$ with equality iff $\bar p_k$ is
constant on $S$; multiplying by $K$ gives the bound. For the second statement, a one-hot
$p_{b,t,\cdot}$ places all its mass on a single index, so an average of $B{\cdot}T$ of them
is supported on at most $B{\cdot}T$ codes, and on at most $K$ codes in any case.
\end{proof}

\section{Computational Cost}
\label{app:cost}

\begin{table}[t]
\centering
\caption{\textbf{Quantizer compute cost.} Forward+backward of the quantizer alone at the
training configuration (batch 32, $T{=}256$, $C{=}64$). At our operating point
$K{=}16{,}384$ the all-code soft assignment is $\sim$10\% of a training step, and LGQ is
$3.8\times$ faster and $1.5\times$ smaller than IBQ at every $K$; IBQ exhausts memory at
$K{=}262{,}144$ where LGQ still trains. FSQ is essentially free and $K$-independent, as it
maintains no codebook. Costs grow linearly in $K$ for all codebook-based methods, so we
scope the efficiency claim to $K \le 65{,}536$.}
\label{tab:quantizer_cost}
\small
\setlength{\tabcolsep}{6pt}
\begin{tabular}{@{}llrrrr@{}}
\toprule
\kilth{Method} & \kilth{$K$} & \kilth{img/s} & \kilth{ms/iter} & \kilth{Peak mem (MB)} & \kilth{Share of step} \\
\midrule
FSQ   & 16{,}384 & 77{,}093 & 0.42 & 1.5 & $<$1\% \\
FSQ   & 65{,}536 & 71{,}933 & 0.45 & 2.3 & $<$1\% \\
\midrule
SimVQ & 16{,}384 & 13{,}123 & 2.44 & 856 & 2.0\% \\
SimVQ & 65{,}536 & 3{,}720 & 8.60 & 3{,}196 & 7.2\% \\
\midrule
IBQ   & 16{,}384 & 694 & 46.09 & 4{,}700 & 38.5\% \\
IBQ   & 65{,}536 & 176 & 182.19 & 18{,}548 & 152\% \\
IBQ   & 262{,}144 & \multicolumn{4}{c}{out of memory} \\
\midrule
\textbf{LGQ} & 4{,}096  & 8{,}925 & 3.59 & 854 & \best{3.0\%} \\
\textbf{LGQ} & 16{,}384 & 2{,}627 & 12.18 & 3{,}165 & \best{10.2\%} \\
\textbf{LGQ} & 65{,}536 & 679 & 47.10 & 12{,}405 & 39.4\% \\
\textbf{LGQ} & 262{,}144 & 168 & 190.33 & 49{,}369 & 159\% \\
\bottomrule
\end{tabular}
\end{table}

\Cref{tab:quantizer_cost} times the quantizer's forward and backward pass in isolation at
the training configuration, which is the measurement \cref{subsec:cost} summarizes. At the
operating point $K{=}16{,}384$ the quantizer accounts for $\sim$$10\%$ of a full optimizer
step; LGQ is $3.8\times$ faster and $1.5\times$ smaller than IBQ at matched $K$, and IBQ
exhausts memory at $K{=}262{,}144$ where LGQ still trains. FSQ is cheaper than either and
$K$-independent, since it has no codebook to score against.

End-to-end the differences compress further. Aggregating \texttt{sacct} elapsed time over
every training job and dividing by epochs completed, one epoch costs $15$--$17$ GPU-hours
for \emph{every} quantizer at $K{=}16{,}384$ (LGQ $16.4$, IBQ $17.3$, SimVQ $15.4$, RotVQ
$15.6$, LFQ $16.8$), and LGQ's own cost rises only from $15.9$ to $17.3$ GPU-hours per
epoch as $K$ grows $16\times$ from $4096$ to $65{,}536$, an $8.9\%$ increase for a
$16\times$ larger codebook. The two measurements are consistent because the quantizer is a
minority of a step at both ends of that range: growing its own share from $3.0\%$ to
$39.4\%$ of a step leaves data loading, the perceptual and adversarial losses, and
per-epoch validation dominating wall-clock, so a large relative change in the quantizer
lands as a single-digit end-to-end one. These aggregate figures are \emph{upper bounds}:
jobs are terminated at a wall-time limit mid-epoch, so each carries a partial epoch of
unattributed work, and they mix training with per-epoch validation.

\FloatBarrier

\section{Baseline Quantizers in Detail}
\label{app:baselines}
This section expands the per-method summary of \cref{sec:related}. \Cref{tab:theory}
contrasts the collapse fix each baseline adopts, the structural cost of that fix, and the
corresponding LGQ design decision; the paragraphs below give the detail behind each row.

\begin{table}[t]
\centering
\caption{\textbf{Theoretical contrast of the five baselines and LGQ.} Each row identifies
the baseline's failure-mode fix, the structural cost of that fix, and the corresponding LGQ
design decision that avoids the cost. The SimVQ row characterizes the published
single-linear reparameterization; our experiments use a stronger nonlinear MLP variant
(\cref{app:quant-config}).}
\label{tab:theory}
\small
\setlength{\tabcolsep}{5pt}
\begin{tabular}{@{}l >{\raggedright\arraybackslash}p{0.24\linewidth} >{\raggedright\arraybackslash}p{0.28\linewidth} >{\raggedright\arraybackslash}p{0.28\linewidth}@{}}
\toprule
\kilth{Method} & \kilth{Collapse fix} & \kilth{Structural cost} & \kilth{LGQ's response} \\
\midrule
VQ      & EMA + argmin assignment    & Only $k^*$ learns; dead codes accumulate              & Soft-assignment $\Rightarrow$ all codes learn \\
RotVQ   & Rotation+rescale STE       & Argmin still picks one code per step                  & Soft gradient avoids needing a rotation \\
FSQ     & Fixed scalar grid          & No data adaptivity, axis-aligned                      & Learnable codes in $\mathbb{R}^C$ \\
SimVQ   & Shared linear basis $\mathbf W$ & Codebook constrained to one learned subspace     & Per-code degrees of freedom retained \\
IBQ     & Softmax STE on all codes   & Unnormalized dot-product logits, fixed $\tau$, entropy penalty & $\ell_2^2$ logits, annealed $\tau$, soft-distribution regularizers \\
\bottomrule
\end{tabular}
\end{table}

\paragraph{VQ-VAE~\cite{vqvae}} The original encoder--quantizer--decoder template uses the
argmin assignment with a straight-through estimator~\cite{ste}: only the selected code
$k^\star$ receives a codebook gradient, so non-selected codes drift into \emph{dead codes},
which is the well-known codebook collapse~\cite{vqvae2,simvq}. EMA updates and codebook
resets mitigate but do not structurally fix. LGQ avoids this by giving every code a
gradient through the soft-assignment.

\paragraph{RotVQ / Rotation Trick~\cite{rotation_trick}} Fifty \etal keep the argmin
forward pass but replace the straight-through gradient with a rotation and rescaling that
aligns the encoder output with its assigned code, preserving their relative angle and
magnitude and recovering information that vanilla STE discards. The fix acts on the
gradient path rather than on how many codes are learned per step, since the assignment
remains hard. LGQ sidesteps this entirely: its soft-assignment gradient already gives every
code a signal proportional to its proximity, with no gradient surgery.

\paragraph{FSQ~\cite{fsq}} Finite Scalar Quantization replaces the learned codebook with a
Cartesian product of fixed scalar levels per channel (\eg $L=[8,8,8,8,4]$ at
$K{=}16{,}384$). Collapse is eliminated by construction, every grid point is a valid code,
so the codebook cannot drift because it cannot move at all, but the trade-off is a loss of
adaptivity: the axis-aligned grid cannot concentrate capacity in high-density regions or
adapt to off-axis cluster structure. LGQ keeps a flat learnable codebook while inheriting
FSQ's collapse-resistance via regularization, gaining adaptivity and full utilization at
once.

\paragraph{LFQ~\cite{lfq}} Lookup-Free Quantization (MAGVIT-v2) encodes each spatial
location as a per-channel sign pattern, giving an implicit codebook of $2^d$ entries kept
balanced by an entropy loss. It scales the implicit codebook well, but restricts codes to
the corners of the $\{-1,+1\}^d$ hypercube, so it cannot place cluster codes off the
sign-pattern lattice. On our backbone this shows up as a slow start: one of the $14$ binary
dimensions saturates to a constant sign from epoch~$8$, halving the reachable codebook to
$\sim$$50\%$ utilization, and the dimension only unsaturates at epoch~$25$, after which
utilization holds at $100\%$. This is why LFQ is reported at a later epoch than the other
rows of \cref{tab:k16k}. LGQ's codes sit anywhere in $\mathbb{R}^C$, removing the lattice
constraint while keeping the entropy-style ``every code should be used'' intuition via
$\mathcal{L}_{\text{div}}$.

\paragraph{SimVQ~\cite{simvq}} Zhu \etal reparameterize the codebook through a single
shared linear layer $\mathbf{C}\mathbf{W}$ ($\mathbf{C}$ a frozen random matrix,
$\mathbf{W}$ the only learnable component), so updating $\mathbf{W}$ moves all $K$ codes
jointly and VQ's disjoint per-code optimization becomes optimization of one linear
subspace. This reaches $100\%$ utilization (reported at $K{=}65{,}536$ on ImageNet), but
binds every code to a single learned basis, limiting the geometry the codebook can
represent. LGQ instead keeps per-code degrees of freedom, with each $\mathbf{e}_k$
independently learnable, while still preventing collapse through explicit usage
regularization ($\lambda_{\text{peak}}{=}\lambda_{\text{div}}{=}0.005$ throughout the $K$
sweep, held fixed rather than retuned per $K$).

\paragraph{VQGAN-LC~\cite{vqganlc}} Closest in spirit to our $K$-scaling experiments,
VQGAN-LC scales the codebook to $100{,}000$ entries at $99\%$ utilization by freezing a
codebook initialized from pretrained visual features and learning only a projector onto
it. Like SimVQ it buys utilization by giving up per-code adaptivity, and like SimVQ its
codes cannot migrate toward the encoder distribution. We do not run it: it requires a
pretrained feature extractor to initialize the codebook, which breaks the matched-backbone,
matched-initialization protocol under which every row of \cref{tab:k16k} is produced, so a
number from it would not be comparable to the others.

\paragraph{IBQ~\cite{ibq}} Index Backpropagation Quantization is closest to LGQ in spirit:
it applies STE to a softmax over $\mathbf{z}^\top\mathbf{C}_k$ logits, so every code
receives a gradient proportional to its probability, exactly as LGQ does, and scales to
$K{=}2^{18}$ with a two-sided (double) quantization loss, a deeper backbone, and a
MAGVIT-v2-style entropy penalty. It differs from LGQ in three ways: unnormalized
dot-product logits instead of squared $\ell_2$; a fixed softmax temperature instead of
annealed $\tau{:}\,1.0\!\to\!0.1$; and an entropy penalty on the assignment distribution
instead of peakedness$+$diversity on the soft-assignment. \Cref{sec:vs-ibq} isolates the
first of these three differences experimentally.

\FloatBarrier

\section{Training Setup}
\label{app:training}
We train on ImageNet ILSVRC-2012~\cite{imagenet} at $256{\times}256$ on H100 GPUs using
DDP across 4 GPUs, with per-GPU batch size 8 (global batch 32). The optimizer is
AdamW~\cite{adamw} with a learning rate $3\times10^{-4}$, no schedule, AMP enabled. Data
augmentation: random resized crop and horizontal flip with ImageNet normalization.

\section{Quantizer Configurations}
\label{app:quant-config}
\textbf{LGQ}: $\lambda_{\text{peak}}=\lambda_{\text{div}}=0.005$, $\tau$ annealed linearly
$1.0\!\to\!0.1$, flattened mode (one codebook of size $K$ over $C$-dimensional vectors,
not per-channel), squared-$\ell_2$ distance, fixed-seed random initialization of code
vectors.
\textbf{RotVQ}: argmin assignment with the rotation trick~\cite{rotation_trick} enabled,
commitment weight $\beta{=}1.0$, and EMA codebook updates with decay $\gamma{=}0.8$.
\textbf{FSQ}: level configuration $[8,8,8,8,4]$, giving $K=16{,}384$.
\textbf{SimVQ}: the shared-codebook reparameterization of~\cite{simvq} with the codebook
transform set to a Linear--ReLU--Linear MLP (hidden width $2C$) and commitment weight
$\beta=10$. This is a deliberate deviation from the published single-linear map
($\mathbf{C}\mathbf{W}$, $\beta{=}1.0$): the more expressive MLP transform improved SimVQ's
reconstruction on our backbone, so we report this stronger variant rather than understate
the baseline.
\textbf{IBQ}: as published~\cite{ibq}: unnormalized dot-product logits
$\mathbf{z}^\top\mathbf{e}_k$ into the softmax, a fixed softmax temperature, and the
standard entropy-based assignment loss.

\section{Compute Budget and Convergence}
\label{app:compute}
Each tokenizer training run uses $4\times$H100 80\,GB on an internal university GPU
cluster with a wall-time budget of two days per resume; long runs are checkpointed and
resumed. MaskGIT training runs use a single H100 (or $4\times$H100 in DDP for the longer
runs). Total compute across all tokenizer and MaskGIT runs reported here is on the order of
a few thousand H100 GPU-hours.

\paragraph{Convergence of the comparison} Each tokenizer is reported at its
best-validation-rFID checkpoint, by which point the per-epoch rFID curves have plateaued;
the differing checkpoint epochs reflect differing convergence rates under a shared
wall-clock budget, not unequal training effort, so a later-epoch checkpoint is not a
better-resourced one. Each MaskGIT transformer is likewise trained to its validation-NLL
plateau, and the two weakest baselines (RotVQ, SimVQ) reach fewer epochs but remain far
behind the leaders, so their stopping point does not affect the ranking.

\section{Selection-Rule Analysis: \texorpdfstring{$\ell_2^2$}{l2-squared} vs.\ Dot-Product (vs.\ IBQ)}
\label{app:vs-ibq}
This section gives the propositions, proofs, and probabilistic interpretation behind the
$\ell_2^2$-vs-dot-product comparison of \cref{sec:vs-ibq}. The closest baseline to LGQ at
the algorithmic level is IBQ~\cite{ibq}, which also applies STE to a softmax over all $K$
codes but scores codes by the unnormalized inner product
$\langle\mathbf{z},\mathbf{e}_k\rangle$ rather than by squared $\ell_2$ distance.

\begin{proposition}[Bayes-optimality of nearest-neighbour selection]
\label{prop:bayes}
For a fixed codebook $\{\mathbf{e}_k\}_{k=1}^K$ and an MSE reconstruction loss
$\|\hat{\mathbf{z}}-\mathbf{z}\|_2^2$, the deterministic hard quantizer that minimizes the
expected reconstruction error is
\[
q^\star(\mathbf{z}) \;=\; \arg\min_k \|\mathbf{z}-\mathbf{e}_k\|_2^2.
\]
\end{proposition}

\begin{proof}
For any quantizer $q$,
$\mathbb{E}\|\mathbf{z}-\mathbf{e}_{q(\mathbf{z})}\|^2
 = \mathbb{E}\,\min_k\|\mathbf{z}-\mathbf{e}_k\|^2
 + \mathbb{E}\!\left(\|\mathbf{z}-\mathbf{e}_{q(\mathbf{z})}\|^2
   - \min_k\|\mathbf{z}-\mathbf{e}_k\|^2\right)$;
the second term is non-negative and vanishes exactly when $q=q^\star$ (Lloyd's optimality
condition~\cite{lloyd1982}).
\end{proof}

LGQ's forward pass is $k^\star=\arg\max_k p_{b,t,k}=q^\star(\mathbf{z}_{b,t})$ because
$p_k$ is monotone in $-\|\mathbf{z}-\mathbf{e}_k\|^2$, so LGQ's selection rule is
Bayes-optimal \emph{for latent-space MSE}, whereas IBQ's is not. \Cref{prop:bayes} concerns
the latent reconstruction term $\|\hat{\mathbf{z}}-\mathbf{z}\|_2^2$ in isolation, for a
\emph{fixed} codebook. LGQ is trained under a composite objective that additionally
includes pixel, perceptual, adversarial, peakedness, and diversity terms, over a codebook
that is itself being learned. The proposition therefore motivates the Euclidean assignment
rule; it does not establish that LGQ is optimal for the objective actually optimized, and
it does not by itself account for the observed rFID gap with IBQ. We treat that gap as
empirical evidence, isolated by the matched logit ablation rather than by this proposition.

\begin{corollary}[Norm-bias gap of dot-product selection]
\label{cor:normbias}
Let $q_{\mathrm{IBQ}}(\mathbf{z}) = \arg\max_k \langle\mathbf{z},\mathbf{e}_k\rangle$.
Because
$\arg\min_k\|\mathbf{z}-\mathbf{e}_k\|^2
   = \arg\min_k\!\bigl(-2\langle\mathbf{z},\mathbf{e}_k\rangle +\|\mathbf{e}_k\|^2\bigr)$,
$q_{\mathrm{IBQ}}=q^\star$ iff $\|\mathbf{e}_k\|^2$ is constant in $k$. When codebook
norms vary, $q_{\mathrm{IBQ}}$ systematically prefers larger-norm codes, with per-sample
excess distortion
\begin{align*}
\Delta_{\mathrm{IBQ}}(\mathbf{z})
 &\;=\; \|\mathbf{z}-\mathbf{e}_{q_{\mathrm{IBQ}}(\mathbf{z})}\|^2
       - \|\mathbf{z}-\mathbf{e}_{q^\star(\mathbf{z})}\|^2 \\
 &\;\le\;
 \max_{j,k}\bigl(\|\mathbf{e}_j\|^2-\|\mathbf{e}_k\|^2\bigr).
\end{align*}
\end{corollary}

\begin{proof}
Expanding $\|\mathbf{z}-\mathbf{e}_k\|^2 = \|\mathbf{z}\|^2
 - 2\langle\mathbf{z},\mathbf{e}_k\rangle + \|\mathbf{e}_k\|^2$,
$q_{\mathrm{IBQ}}$ matches $q^\star$ when the $\|\mathbf{e}_k\|^2$ term is constant in $k$
and is otherwise dominated by the largest-norm code. The substitution error is bounded by
the maximum norm gap.
\end{proof}

Nothing in IBQ's loss explicitly constrains $\|\mathbf{e}_k\|$ to be constant across $k$,
so the gap is a structural property of the rule rather than an implementation detail, and
it grows with codebook heterogeneity. The bound above is worst-case over $\mathbf{z}$ and
independent of it, so it upper-bounds rather than predicts the typical excess distortion.
The realized effect is therefore measured rather than derived, by the matched logit
ablation of \cref{sec:vs-ibq}, which holds backbone, seed, schedule and budget fixed and
varies only the quantity the softmax scores. That experiment, rather than the IBQ--LGQ rFID
gap of \cref{tab:k16k}, is what isolates the selection rule: the published IBQ
configuration also differs in its two-sided (double) quantization loss, backbone depth, and
entropy penalty.

\paragraph{Measured norm bias}
The matched logit ablation lets us test \cref{cor:normbias} directly, since both arms share
a seed and an initialization and differ only in the scored quantity. We identify dead codes
by the batch-averaged soft assignment $\bar p_k$ stored by the quantizer; in the
dot-product arm this statistic is bimodal, around $10^{-14}$ against above $10^{-6}$, so a
$10^{-8}$ cut is insensitive to its exact value. The dead set is fixed from the fourth
epoch onward and never changes membership (\cref{tab:normbias}).

\begin{table}[t]
\centering
\caption{\textbf{Codebook norms in the dot-product arm.} The same $3{,}300$ codes are dead
at every checkpoint. Dead and live norms are disjoint, and the separation widens
monotonically: every one of the $3{,}300$ shrank between consecutive checkpoints, and the
dead ranges at successive epochs do not overlap. The last column is the correlation between
a code's norm and the log of its usage.}
\label{tab:normbias}
\small
\begin{tabular}{@{}lcccc@{}}
\toprule
\kilth{Epoch} & \kilth{Dead} & \kilth{Dead $\|\mathbf{e}_k\|$} & \kilth{Live $\|\mathbf{e}_k\|$} & \kilth{$\mathrm{corr}$} \\
\midrule
$18$ & $3{,}300$ & $[0.823,\,1.022]$ & $[5.273,\,5.909]$ & $+0.996$ \\
$20$ & $3{,}300$ & $[0.648,\,0.805]$ & $[5.232,\,5.859]$ & $+0.997$ \\
$22$ & $3{,}300$ & $[0.511,\,0.634]$ & $[5.198,\,5.848]$ & $+0.997$ \\
\bottomrule
\end{tabular}
\end{table}

This is the mechanism of \cref{cor:normbias} operating during training rather than at a
fixed codebook: a code that falls behind on magnitude loses assignments, and without
assignments it receives no gradient with which to grow back, so the norm gap compounds. The
$\ell_2^2$ arm is the control. It has no dead codes at all, and the correlation between
norm and usage is $+0.08$, despite carrying a $3.6\times$ spread in norm
($\|\mathbf{e}_k\|\in[2.80,10.01]$ at epoch~$23$). The squared distance therefore does not
require a norm-homogeneous codebook to use every code, whereas the inner product does. We
report this as a property of these two runs; it is a measurement of the predicted effect,
not an independent proof of the corollary.

\paragraph{Scaling across $K$: does the baseline degrade where LGQ does not?}
A natural question about the codebook-size scaling of \cref{sec:results-sweep} is whether
it is specific to LGQ or would appear for any quantizer. We therefore trained IBQ at
$K{=}4{,}096$, $16{,}384$ and $65{,}536$ under the identical protocol (same encoder,
decoder, discriminator, optimizer, effective batch, seed and data order) and compare at a
\emph{matched} epoch at each $K$, taking every value from the training-time validation log
of the run itself. \Cref{tab:ibq-kscale} reports epoch~$24$ at $K{=}4{,}096$, epoch~$45$
at $K{=}16{,}384$ and epoch~$23$ at $K{=}65{,}536$, in each case the deepest epoch at which
both runs have a logged rFID.

\begin{table}[t]
\centering
\caption{\textbf{LGQ against IBQ across codebook size, at matched budget.} The two methods
are compared at a matched epoch at each $K$: epoch~$24$ at $K{=}4{,}096$, epoch~$45$ at
$K{=}16{,}384$, and epoch~$23$ at $K{=}65{,}536$. Every value is read from the run's own
training-time validation log, so the two methods are measured identically. Neither method
has reached the budget of the main comparison at these epochs and both continue to
improve, so these are budget-matched rather than final numbers, and they are not
comparable to the main-paper tables. Each row is a single training run. ``Active'' is the
number of codes assigned at least once. \best{Bold} marks the better of the two at each
$K$.}
\label{tab:ibq-kscale}
\small
\setlength{\tabcolsep}{4pt}
\begin{tabular}{@{}llccccccr@{}}
\toprule
\kilth{$K$} & \kilth{Method} & \kilth{PSNR $\uparrow$} & \kilth{SSIM $\uparrow$} & \kilth{LPIPS $\downarrow$} & \kilth{Rec.\ loss $\downarrow$} & \kilth{rFID $\downarrow$} & \kilth{Util.\ \%} & \kilth{Active} \\
\midrule
\multirow{2}{*}{$4{,}096$}
 & IBQ & $21.190$ & $0.5999$ & $0.2852$ & $0.2437$ & $20.68$ & $99.95$ & $4{,}094$ \\
 & LGQ & $\mathbf{21.409}$ & $\mathbf{0.6065}$ & $\mathbf{0.2705}$ & $\mathbf{0.2343}$ & $\mathbf{16.14}$ & $\mathbf{100.00}$ & $\mathbf{4{,}096}$ \\
\midrule
\multirow{2}{*}{$16{,}384$}
 & IBQ & $21.492$ & $0.6162$ & $0.2614$ & $0.2351$ & $16.63$ & $99.56$ & $16{,}312$ \\
 & LGQ & $\mathbf{21.738}$ & $\mathbf{0.6260}$ & $\mathbf{0.2436}$ & $\mathbf{0.2230}$ & $\mathbf{12.89}$ & $\mathbf{100.00}$ & $\mathbf{16{,}384}$ \\
\midrule
\multirow{2}{*}{$65{,}536$}
 & IBQ & $21.637$ & $0.6158$ & $0.2593$ & $0.2279$ & $16.10$ & $55.07$ & $36{,}088$ \\
 & LGQ & $\mathbf{22.093}$ & $\mathbf{0.6460}$ & $\mathbf{0.2304}$ & $\mathbf{0.2145}$ & $\mathbf{10.82}$ & $\mathbf{96.84}$ & $\mathbf{63{,}465}$ \\
\bottomrule
\end{tabular}
\end{table}

LGQ is ahead on every metric at all three codebook sizes, so the scaling behaviour is not
an artifact of how we selected checkpoints for \cref{tab:lgq_sweep}. The utilization
column is the more informative one, because the two methods move in opposite directions as
training proceeds. At $K{=}4{,}096$ and $K{=}16{,}384$ IBQ holds $99.95\%$ and $99.56\%$,
so its anti-collapse machinery works at small and medium $K$. At $K{=}65{,}536$ it peaks at
$98.18\%$ in epoch~$2$ and then falls, to $67.33\%$ by epoch~$12$ and $55.07\%$ ($36{,}088$
of $65{,}536$ codes) by epoch~$23$. LGQ at the same $K$ moves the other way, rising from
$82.64\%$ at epoch~$12$ through $89.96\%$ at epoch~$18$, $95.19\%$ at epoch~$22$ and
$96.84\%$ at epoch~$23$ to $100\%$ at epoch~$29$, after which it stays there. IBQ therefore
begins with the higher utilization and ends far below it: the two curves cross. We read
this as the entropy penalty being tuned at a codebook size where its temperature is
appropriate, and losing codes once $K$ grows past that point, whereas LGQ's diversity term
is stated relative to $K$ and needs no retuning. We report this from one run per cell and
do not claim it generalizes to other baselines, which we have not swept across $K$.

\paragraph{Probabilistic interpretation}
LGQ's softmax over $-\|\mathbf{z}-\mathbf{e}_k\|^2/\tau$ is the posterior of a Gaussian
mixture with isotropic covariance $\sigma^2=\tau/2$; the hard-assignment rule is therefore
MAP inference under a GMM. IBQ's softmax over $\langle\mathbf{z},\mathbf{e}_k\rangle$ is the
posterior of a von Mises--Fisher mixture on the hypersphere, which is the correct model
\emph{only} when features are constrained to unit norm. Since both our backbone and IBQ's
produce unconstrained $\mathbb{R}^C$ features and the decoder consumes Euclidean inputs,
the vMF prior is misspecified for both the data and the loss, so LGQ's distance-based
scoring matches the geometry that the rest of the pipeline actually uses. The dot-product
rule does have two genuine advantages, so the trade-off is worth stating explicitly: its
encoder gradient
$\partial\langle\mathbf{z},\mathbf{e}_k\rangle/\partial\mathbf{z}=\mathbf{e}_k$ is bounded
by $\|\mathbf{e}_k\|$ (whereas
$\partial\|\mathbf{z}-\mathbf{e}_k\|^2/\partial\mathbf{z}=2(\mathbf{z}-\mathbf{e}_k)$ grows
with the encoder--codebook gap), and it costs a single BLAS matmul without the
squared-difference expansion. LGQ accepts both costs, spending an annealed $\tau$ to keep
the backward signal bounded (the annealed-softmax Jacobian bound of \cref{sec:prelim}) and
paying for a pairwise distance, in exchange for a selection rule whose forward pass is the
Bayes-optimal nearest-neighbour rule for the reconstruction loss (\cref{prop:bayes}). The
controlled isolation of the selection rule alone is the matched logit ablation of
\cref{sec:vs-ibq}.

\FloatBarrier

\section{Per-Epoch Reconstruction Curves}
\label{app:curves}
Every tokenizer in this paper is reported at its best-validation-rFID checkpoint under a
shared wall-clock budget, so the natural question is whether the ordering those checkpoints
induce is an artifact of where each run happened to be stopped.
\Cref{fig:order_stability} plots validation rFID against training epoch for five of the six
quantizers of the main comparison, one run each, over a common window ending at
epoch~$36$, and answers it directly: the curves separate into three bands that never meet,
and the single pair that does exchange places is a pair whose runs are not separable at
this budget.

\begin{figure}[t]
\centering
\includegraphics[width=0.75\linewidth]{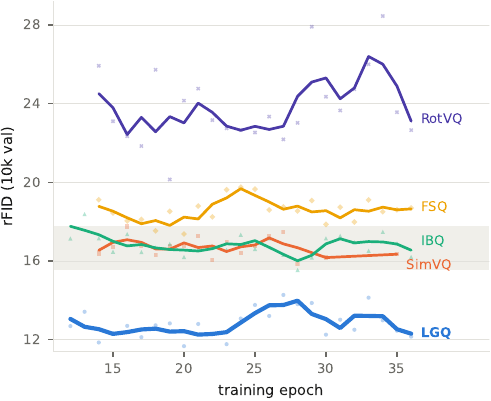}
\caption{\textbf{Reconstruction rFID against training epoch} for five of the six quantizers
of the main comparison over the epochs. Lines are a three-epoch centred rolling mean and
the faint markers are the raw per-epoch values. Of the $10$ method pairs measured at four
or more common epochs, $9$ never invert. LGQ leads at all $13$ epochs at which it was
measured in this window, with a minimum margin of $1.74$ rFID (mean $3.36$) against a
within-run epoch-to-epoch scatter of $0.60$. rFID here is the $10{,}000$-image validation
metric logged during training; the $50{,}000$-image ordering was checked separately at two
checkpoints and was identical at both.}
\label{fig:order_stability}
\end{figure}

\section{Codebook Geometry}
\label{app:geometry}
\Cref{fig:umap} gives the UMAP overlay of the learned codes on the continuous encoder
distribution referenced in \cref{sec:results-geometry}; the tokenizer architecture is
diagrammed in \cref{fig:arch}.

\begin{figure}[t]
    \centering
    \includegraphics[width=0.85\linewidth]{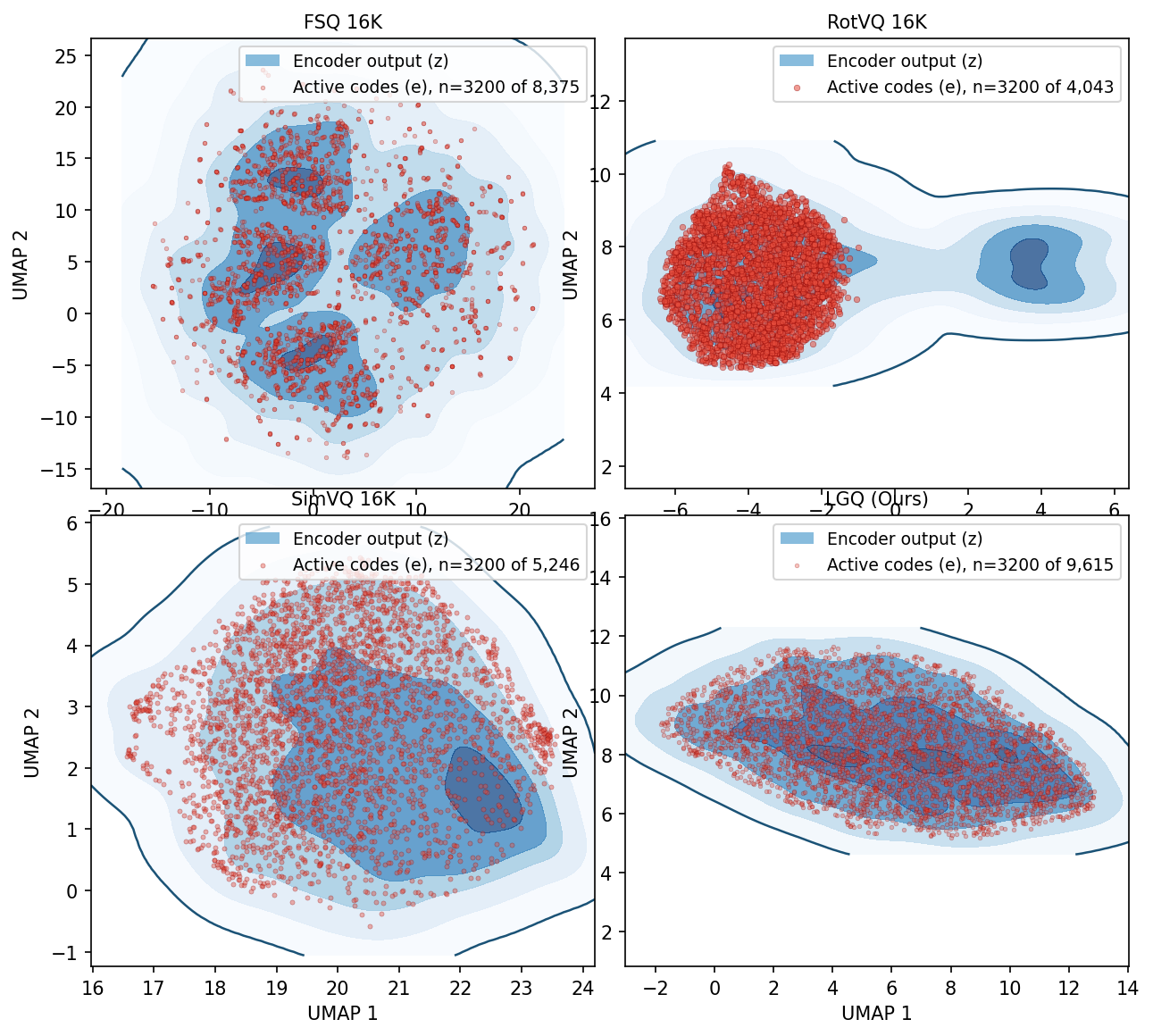}
    \caption{\textbf{UMAP of encoder latents (gray) with active codebook entries (red) at
    $K=16{,}384$.} LGQ's codes spread across the encoder distribution, consistent with its
    $100\%$ utilization; RotVQ's cluster in a small subregion, the visual signature of
    under-utilization. The overlay geometry tracks the rFID ranking of \cref{tab:k16k}.}
    \label{fig:umap}
\end{figure}

\FloatBarrier

\section{Full Qualitative Comparisons}
\label{app:qualitative}
\Cref{fig:recon-full,fig:gen-full} give the complete qualitative grids summarized by the
single-example strip in \cref{fig:recon}. \Cref{fig:recon-full} shows reconstructions of
six ImageNet validation images for every tokenizer; \cref{fig:gen-full} shows
class-conditional MaskGIT samples across eight ImageNet classes for all five tokenizers.
Because the cells of \cref{fig:gen-full} are chosen by classifier confidence, we also give
\cref{fig:gen-random}, an unselected draw in which every generated sample is shown; the two
together separate what the tokenizers can produce from what they typically produce.

\begin{figure}[t]
\centering
\includegraphics[width=0.9\linewidth]{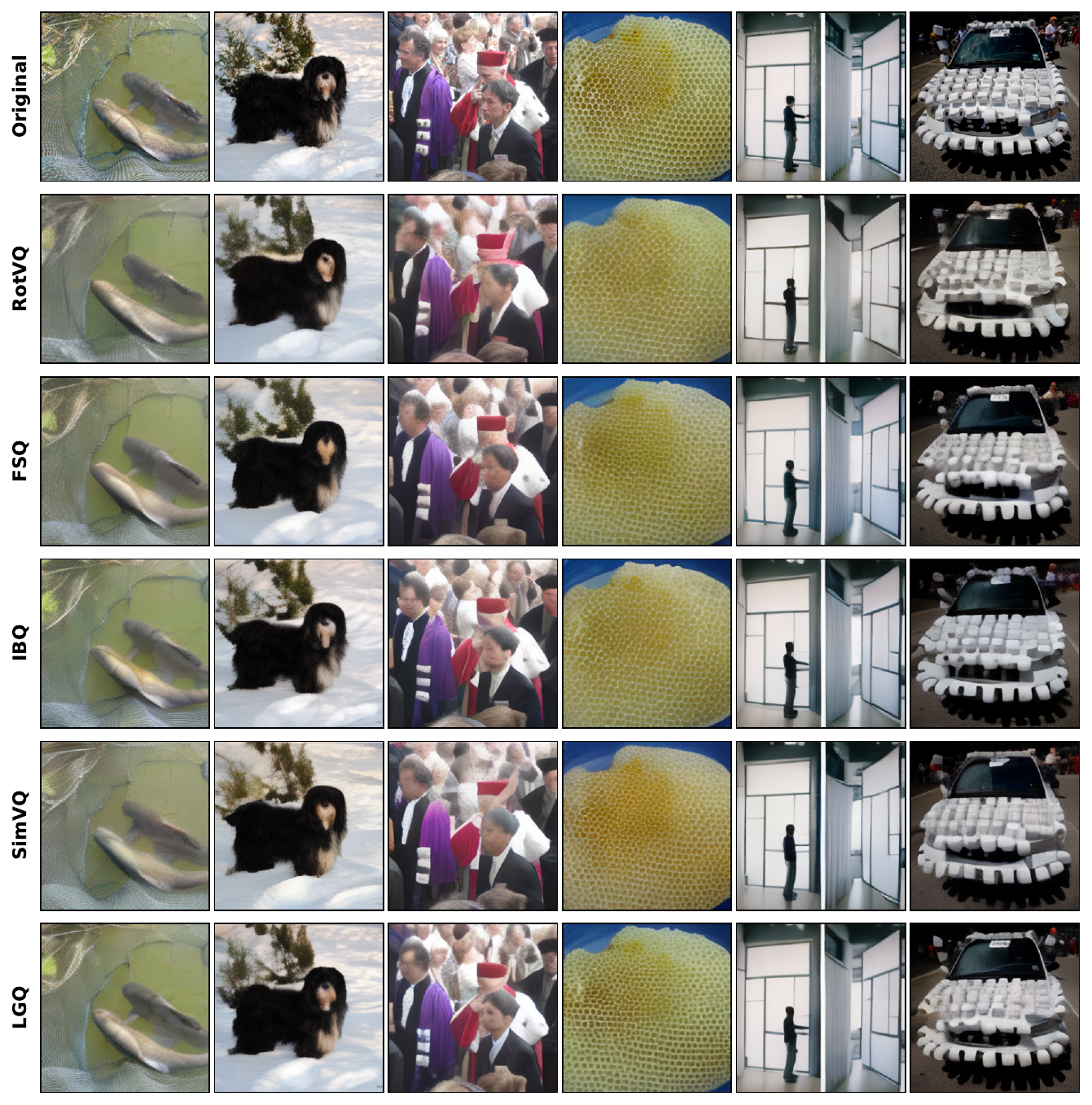}
\caption{\textbf{Reconstructions at $K{=}16{,}384$ (full grid).} Top row: original ImageNet
validation images; each subsequent row is the reconstruction from one frozen tokenizer
(best-rFID checkpoint). LGQ's reconstructions retain the finest texture and color
fidelity, consistent with its best rFID/LPIPS in \cref{tab:k16k}; RotVQ shows the most
blur, matching its higher rFID. \Cref{fig:recon} uses the second column (dog) of this
grid.}
\label{fig:recon-full}
\end{figure}

\begin{figure}[t]
\centering
\includegraphics[width=\linewidth]{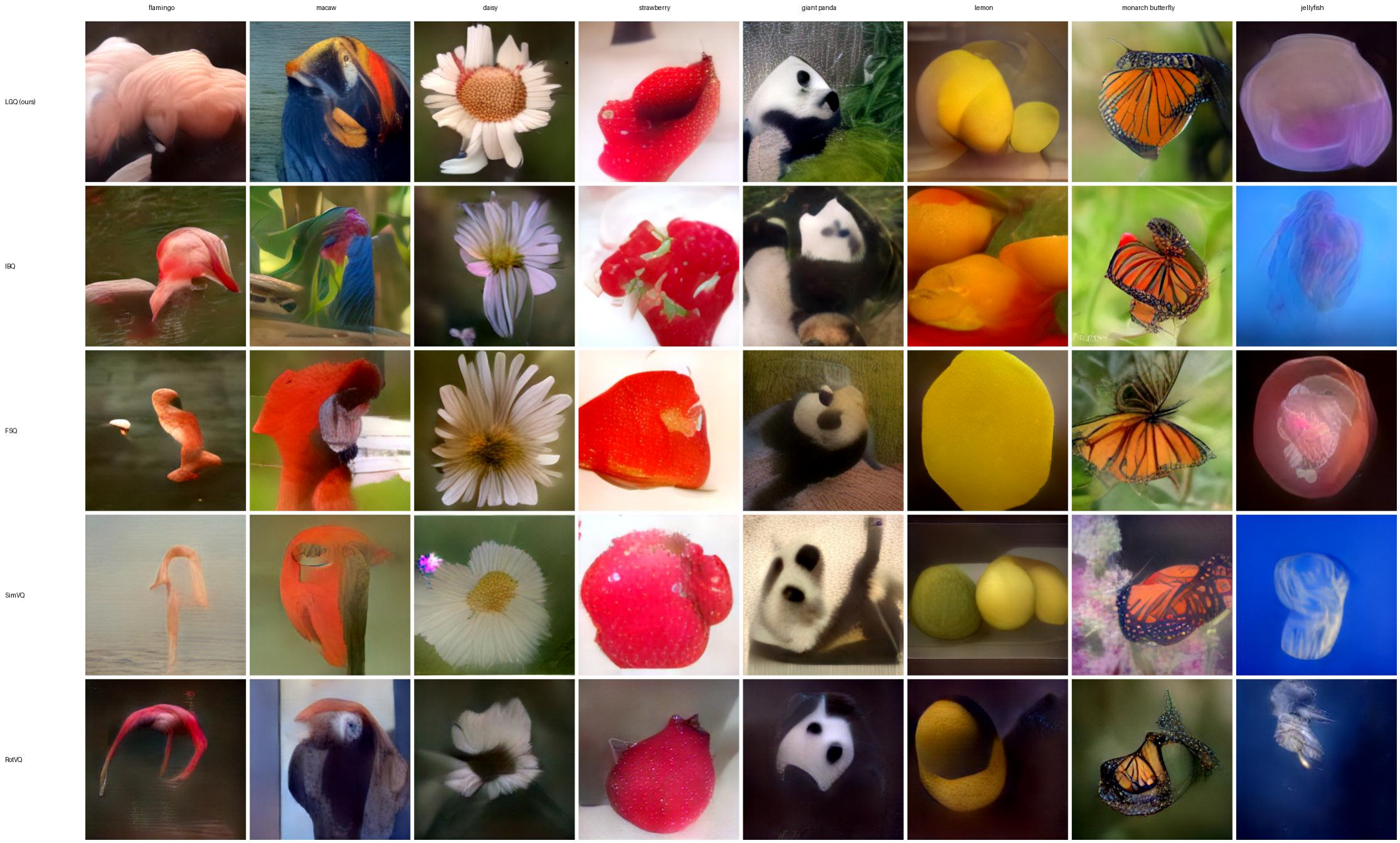}
\caption{\textbf{Class-conditional MaskGIT samples at $K{=}16{,}384$ (full grid).} Rows are
tokenizers ordered by gFID (top to bottom: LGQ, IBQ, FSQ, SimVQ, RotVQ), columns are eight
ImageNet classes; each cell is the highest-confidence generation for that class, selected
by a pretrained ImageNet classifier from the same $50$K-sample set used for the gFID in
\cref{tab:maskgit_gen}. Quality degrades top to bottom in step with gFID: LGQ's samples are
the sharpest and most class-coherent (best gFID $57.69$), while RotVQ's are the blurriest
(gFID $92.99$). Because the cells are classifier-selected, this grid is a best-case display
for every method; \cref{fig:gen-random} shows an unselected draw.}
\label{fig:gen-full}
\end{figure}

\begin{figure}[t]
\centering
\includegraphics[width=\linewidth]{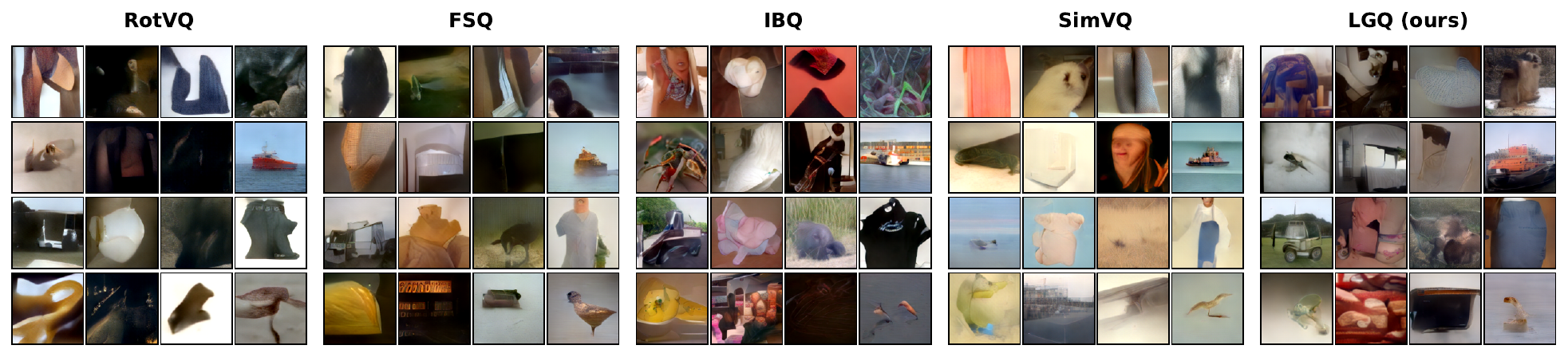}
\caption{\textbf{Unselected class-conditional MaskGIT samples at $K{=}16{,}384$.} Panels
are the same five tokenizers as \cref{fig:gen-full} (left to right: RotVQ, FSQ, IBQ, SimVQ,
LGQ), sixteen samples each. Class labels are drawn uniformly at random under a fixed seed
and reused across panels, so each grid position holds the same class in every panel.
\emph{Every generated sample is shown}: no filtering and no confidence ranking, in contrast
to \cref{fig:gen-full}. At this backbone scale ($27.99$M encoder/decoder parameters,
$2.56\times$ smaller than the standard VQ-GAN $f16$) absolute sample quality is low for
every method and only RotVQ is clearly distinguishable by eye, consistent with its gFID of
$92.99$; LGQ, IBQ, SimVQ and FSQ are not visually separable at this draw. We therefore rest
the generation comparison on the gFID over $50$K samples of \cref{tab:maskgit_gen} rather
than on qualitative panels.}
\label{fig:gen-random}
\end{figure}

\FloatBarrier

\end{document}